\pdfoutput=1

\def\MODE{1} 
\if\MODE1

\else

\fi

\documentclass[10pt]{article}

\usepackage{palatcm}
\usepackage{graphicx}
\usepackage{subfigure} 
\usepackage{soul}
\usepackage{color, xcolor}
\usepackage[thinlines]{easytable}
\usepackage{relsize}
\usepackage{xfrac}
\usepackage{verbatim}
\usepackage{algorithm}
\usepackage[noend]{algorithmic}
\usepackage{amsmath}
\usepackage{amssymb}
\usepackage{amsthm}
\usepackage{epsfig}
\usepackage{wrapfig}
\usepackage{url}
\usepackage[colorlinks=true,citecolor=blue,linkcolor=blue]{hyperref}
\usepackage{multirow}
\usepackage{fullpage}

\usepackage[T1]{fontenc}
\usepackage{upgreek}

\usepackage[greek,english]{babel}
\usepackage[utf8x]{inputenc}
\usepackage{pifont}
\usepackage{float}

\usepackage{hyperref}       
\usepackage{url}            
\usepackage{booktabs}       
\usepackage{amsfonts}       
\usepackage{microtype}      
\usepackage{microtype}
\usepackage{graphicx}
\usepackage{subfigure}
\usepackage{booktabs} 
\usepackage{amsthm,amsmath,amssymb}
\usepackage{float,url,amsfonts,alltt}
\usepackage{mathtools,rotating}
\usepackage{ifpdf,fancyvrb}
\usepackage{enumitem}

\usepackage{microtype}
\usepackage{graphicx}
\usepackage{subfigure}
\usepackage{booktabs} 
\usepackage{hyperref}

\usepackage{graphicx,xspace,verbatim,comment}
\usepackage{hyperref,array,color,balance,multirow}
\usepackage{balance,float,url,amsfonts,alltt}
\usepackage{mathtools,rotating,amsmath,amssymb}
\usepackage{color,ifpdf,fancyvrb,array}
\usepackage{etoolbox,listings}
\usepackage{bigstrut,morefloats}

\usepackage{pbox}
\usepackage[boxruled,algo2e]{algorithm2e}
\DeclarePairedDelimiterX{\inp}[2]{\langle}{\rangle}{#1, #2}

\newtheorem{theorem}{Theorem}

\newtheorem{remark}{Remark}
\newtheorem{corollary}[theorem]{Corollary}
\newtheorem{lemma}[theorem]{Lemma}
\newtheorem{definition}{Definition}
\newcommand{\eat}[1]{}

\newcommand{\ie}{{\emph i.e.,} }
\newcommand{\eg}{{\it e.g.}, }

\DeclarePairedDelimiter{\ceil}{\lceil}{\rceil}

\numberwithin{equation}{section}

\newcommand{\draco}{\textsc{Draco}}

\newcommand{\WH}[1]{}
\newcommand{\ST}[1]{}
\newcommand{\XP}[1]{}

\newtoggle{tr}
\toggletrue{tr}
\iftoggle{tr}{
	\makeatother
}{}

\title{ \draco{}:\\ Byzantine-resilient Distributed Training via  Redundant Gradients}

\author{
Lingjiao Chen, Hongyi Wang, Zachary Charles, Dimitris Papailiopoulos\\
University of Wisconsin-Madison
} 

\date{}

\begin{document}

\maketitle

\begin{abstract}
Distributed model training is vulnerable to byzantine system failures and adversarial compute nodes, \ie nodes that use malicious updates to corrupt the global model stored at a parameter server (PS). 
To guarantee some form of robustness, recent work suggests using variants of the geometric median as an aggregation rule, in place of gradient averaging. 
Unfortunately,   median-based  rules can incur a prohibitive computational overhead in large-scale settings, and their convergence guarantees often require  strong assumptions. In this work, we present \draco{}, a scalable framework for robust distributed training that uses ideas from coding theory.
In \draco{}, each compute node evaluates redundant gradients that are  used by the parameter server to eliminate the effects of adversarial updates.
\draco{} comes with problem-independent robustness guarantees, and  the model that it trains is   identical to the one trained in the adversary-free setup. 
We provide extensive experiments on real datasets and distributed setups across a variety of large-scale models, where we show that \draco{} is several times, to orders of magnitude faster than median-based approaches.
\end{abstract}

\section{Introduction}\label{Sec:Intro}
Distributed and parallel implementations of stochastic optimization algorithms have become the de facto standard in large-scale model training~\cite{ParameterServer,recht2011hogwild,Syn_DMLSGD_12,Syn_DML_10,tensorflow, MXNET, PyTorch, MS_Adam}.
Due to increasingly common malicious attacks, hardware and software errors \cite{ByzantineError_1999,Byzantine2007,ByzantineML_NIPS17,ByzantineML_SIGMETRICS18}, protecting distributed machine learning against adversarial attacks and failures has become increasingly important.
Unfortunately, even a single adversarial node in a distributed setup can introduce arbitrary bias and inaccuracies to the end model\cite{ByzantineML_NIPS17}.

A recent line of work \cite{ByzantineML_NIPS17,ByzantineML_SIGMETRICS18} studies this problem under a synchronous training setup, where compute nodes evaluate gradient updates and ship them to a parameter server (PS) which stores and updates the global model.
Many of the aforementioned work use median-based aggregation, including the geometric median (GM) instead of averaging in order to make their computations more robust.
The advantage of median-based approaches is that they can be robust to up to a constant fraction of the compute nodes being adversarial \cite{ByzantineML_SIGMETRICS18}. 
However, in large data settings, the cost of computing the geometric median can dwarf the cost of computing a batch of gradients \cite{ByzantineML_SIGMETRICS18}, rendering it impractical.
Furthermore, proofs of convergence for such systems require restrictive assumptions such as convexity, and need to be re-tailored to each different training algorithm.
A scalable distributed training framework that is robust against adversaries and can be applied to a large family of training algorithms (\eg mini-batch SGD, GD, coordinate descent, SVRG, etc.) remains an open problem.

\begin{figure}[t]
	\centering
	\includegraphics[width=0.55\linewidth]{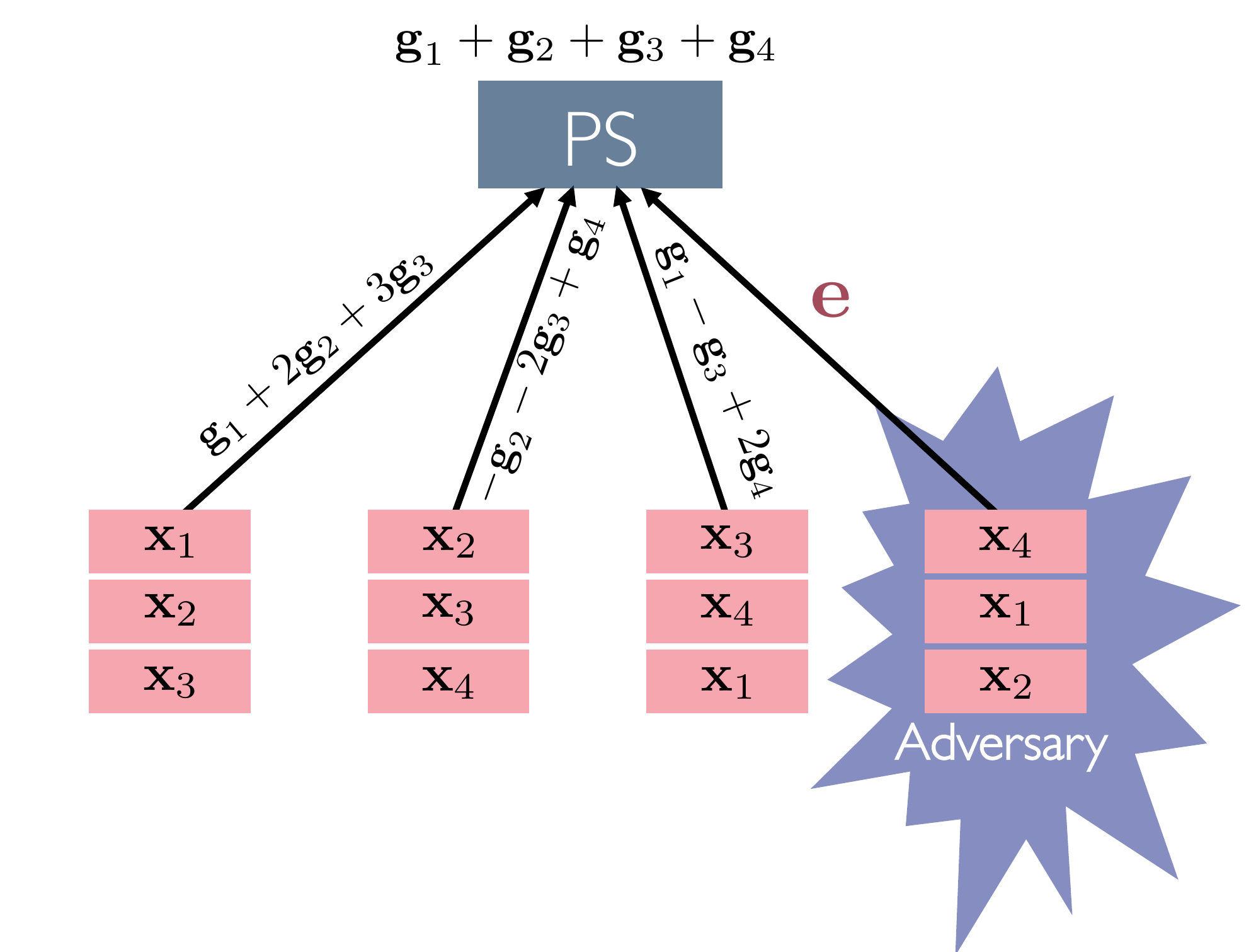}
	\vspace{-0.3cm}
	\caption{The high level idea behind \draco{}'s algorithmic redundancy. Suppose we have $4$ data points ${\bf x}_1, \ldots, {\bf x}_4$, and let ${\bf g}_i$ be  the gradient of the model with respect to data point ${\bf x}_i$.
		Instead of having each compute node $i$ evaluate a single gradient ${\bf g}_i$, \draco{} assigns each node redundant gradients. 
		In this example, the replication ratio is $3$, and the parameter server can recover the sum of the gradients from any $2$ of the encoded gradient updates. Thus, the PS can still recover the sum of gradients in the presence of an adversary. This can be done through a majority vote on all $6$ pairs of encoded gradient updates. This intuitive idea does not scale to a large number of compute nodes. \draco{} implements a more systematic and efficient encoding and decoding mechanism that scales to any number of machines.}
	\label{fig:RunningExample}
\end{figure}

In this paper, we instead use ideas from coding theory to ensure robustness during distributed training.
We present \draco{}, a general distributed training framework that is robust against  adversarial  nodes and worst-case compute errors. 
We show that \draco{} can resist any $s$ adversarial compute nodes during training and returns a model {\it identical} to the one trained in the adversary-free setup.
This allows \draco{} to come with ``black-box'' convergence guarantees, \ie proofs of convergence in the adversary-free setup carry through to the adversarial setup with no modification, unlike prior median-based approaches \cite{ByzantineML_NIPS17,ByzantineML_SIGMETRICS18}.
Moreover, in median-based approaches such as \cite{ByzantineML_NIPS17,ByzantineML_SIGMETRICS18}, the median computation may dominate the overall training time. In \draco{}, most of the computational effort is carried through by the compute nodes. This key factor allows our framework to offer up to orders of magnitude faster convergence in real distributed setups. 

To design \draco{}, we borrow ideas from coding theory and algorithmic redundancy.
In standard adversary-free distributed computation setups, during each distributed round, each of the $P$ compute nodes processes $B/P$ gradients and ships their sum to the parameter server. 
In \draco{}, each compute node processes $rB/P$ gradients and sends a linear combination of those to the PS. Thus, \draco{} incurs a computational redundancy ratio of $r$.
While this may seem sub-optimal, we show that under a worst-case adversarial setup, it is information--theoretically impossible to design a system that obtains identical models to the adversary--free setup  with less redundancy.
Upon receiving the $P$ gradient sums, the PS uses a ``decoding'' function to remove the effect of the adversarial nodes and reconstruct the original desired sum of the $B$ gradients. 
With redundancy ratio $r$, we show that \draco{} can tolerate up to $(r-1)/2$ adversaries, which is information--theoretically {\it tight}. See Fig.~\ref{fig:RunningExample} for a toy example of \draco{}'s functionality.

We present two encoding and decoding techniques for \draco{}. 
The encoding schemes are based on the fractional repetition code and cyclic repetition code presented in \cite{Gradient_Coding,raviv2017gradient}. 
In contrast to previous work on stragglers and gradient codes \cite{Gradient_Coding,raviv2017gradient,charles2017approximate}, our decoders are tailored to the adversarial setting and use different methods.
Our decoding schemes utilize an efficient majority vote decoder and a novel Fourier decoding technique.

Compared to median-based techniques that can tolerate approximately a constant fraction of ``average case'' adversaries, \draco{}'s  $(r-1)/2$ bound on the number of ``worst-case'' adversaries may be significantly smaller.
However, in realistic regimes where only a constant number of nodes are malicious, \draco{} is significantly faster as we show in experiments in Section \ref{Sec:Experiment}.

We implement \draco{} in PyTorch and deploy it on distributed setups on Amazon EC2, where we compare against median-based training algorithms on several real world datasets and various ML  models.
We show that \draco{} is up to orders of magnitude faster compared to GM-based approaches across a range of neural networks, \eg
LeNet,
VGG-19,
AlexNet,
ResNet-18, and
ResNet-152, and always converges to the correct adversary-free model, while in some cases median-based approaches do not converge. 

\paragraph{Related Work}\label{Sec:Related}
The large-scale nature of modern machine learning has spurred a great deal of novel research on distributed and parallel training algorithms and systems \cite{recht2011hogwild, dean2012large, alistarh2017qsgd, jaggi2014communication, liu2014asynchronous1, mania2015perturbed, chen2016revisiting}.
Much of this work focuses on developing and analyzing efficient distributed training algorithms. This work shares ideas with {\it federated learning}, in which training is distributed among a large number of compute nodes without centralized training data \cite{konevcny2015federated, FederatedLearning, FederatedLearn_Aggregate}.

Synchronous training can suffer from straggler nodes~\cite{zaharia2008improving}, where a few compute nodes are significantly slower than average. While early work on straggler mitigation used techniques such as job replication \cite{shah2016redundant}, more recent work has employed coding theory to speed up distributed machine learning systems \cite{MatrixCode_2016,li2015coded,dutta2016short,dutta2017coded,reisizadeh2017coded,YangGK2017CodedInverse}. 
One notable technique is {\it gradient coding}, a straggler mitigation method proposed in \cite{Gradient_Coding}, which uses codes to speed up synchronous distributed first-order methods \cite{raviv2017gradient,charles2017approximate,cotter2011better}. Our work builds on and extends this work to the adversarial setup \cite{Draco2018,chendraco}. Mitigating adversaries can often be more difficult than mitigating stragglers since in the adversarial setup we have no knowledge as to which nodes are the adversaries.

The topic of byzantine fault tolerance has been extensively studied since the early 80s~\cite{lamport1982byzantine}.
There has been substantial amounts of work recently on byzantine fault tolerance in distributed training which shows that while average-based gradient methods are susceptible to adversarial nodes~\cite{ByzantineML_NIPS17,ByzantineML_SIGMETRICS18}, median-based update methods can achieve good convergence while being robust to adversarial nodes. Both \cite{ByzantineML_NIPS17} and \cite{ByzantineML_SIGMETRICS18} use variants of the geometric median to improve the tolerance of first-order methods against adversarial nodes. Unfortunately, convergence analyses of median approaches often require restrictive assumptions and algorithm-specific proofs of convergence. Furthermore, the geometric median aggregation may dominate the training time in large-scale settings.

The idea of using redundancy to guard against failures in computational systems has existed for decades. Von Neumann used redundancy and majority vote operations in boolean circuits to achieve accurate computations in the presence of noise with high probability \cite{von1956probabilistic}. These results were further extended in work such as \cite{pippenger1988reliable} to understand how susceptible a boolean circuit is to randomly occurring failures. Our work can be seen as an application of the aforementioned concepts to the context of distributed training in the face of adversity.
\section{Preliminaries}\label{Sec:Preli}

\paragraph{Notation}
In the following, we denote matrices and vectors in bold, and scalars and functions in standard script. We let $\mathbf{1}_m$ denote the $m\times 1$ all ones vector, while $\mathbf{1}_{n\times m}$ denotes the all ones $n\times m$ matrix. We define $\mathbf{0}_{m}, \mathbf{0}_{n\times m}$ analogously.
Given a matrix $\mathbf{A} \in \mathbb{R}^{n\times m}$, we let $\mathbf{A}_{i,j}$ denote its entry at location $(i,j)$, $\mathbf{A}_{i,\cdot} \in \mathbb{R}^{1\times m}$ denote its $i$th row, and $\mathbf{A}_{\cdot, j} \in \mathbb{R}^{n\times 1}$ denote its $j$th column. Given $S \subseteq \{1,\ldots, n\}$, $T \subseteq \{1,\ldots, m\}$, we let $\mathbf{A}_{S,T}$ denote the submatrix of $\mathbf{A}$ where we keep rows indexed by $S$ and columns indexed by $T$. Given matrices $\mathbf{A}, \mathbf{B} \in \mathbb{R}^{n\times m}$, their {\it Hadamard product}, denoted $\mathbf{A} \odot \mathbf{B}$, is defined as the $n \times m$ matrix where $(\mathbf{A}\odot \mathbf{B})_{i,j} = \mathbf{A}_{i,j}\mathbf{B}_{i,j}$.

\paragraph{Distributed Training}
The process of training a model from data can be cast as an optimization problem known as {\it empirical risk minimization} (ERM): 
$$\min_{{\bf w}} \frac{1}{n}\sum_{i=1}^n \ell({\bf w};\mathbf{x}_i)$$
where $\mathbf{x}_i\in\mathbb{R}^m$ represents the $i$th data point, $n$ is the total number of data points, ${\bf w}\in\mathbb{R}^d$ is a model, and $\ell(\cdot;\cdot)$ is a loss function that measures the accuracy of the predictions made by the model on each data point. 

One way to approximately solve the above ERM is through stochastic gradient descent (SGD), which operates as follows. We initialize the model at an initial point ${\bf w}_0$ and then iteratively update it according to
$${\bf w}_k = {\bf w}_{k-1} - \gamma \nabla\ell({\bf w}_{k-1}; \mathbf{x}_{i_k}),$$
where $i_k$ is a random data-point index sampled from $\{1,\ldots, n\}$, and $\gamma>0$ is the learning rate.

In order to take advantage of distributed systems and parallelism, we often use {\it mini-batch} SGD. At each iteration of mini-batch SGD, we select a random subset $S_k \subseteq \{1,\ldots, n\}$ of the data and update our model according to
$${\bf w}_k = {\bf w}_{k-1} - \dfrac{\gamma}{|S_k|} \sum_{i \in S_k}\nabla\ell({\bf w}_{k-1}; \mathbf{x}_{i}).$$
Many distributed versions of mini-batch SGD partition the gradient computations across the compute nodes. After computing and summing up their assigned gradients, each nodes sends their respective sum back to the PS. The PS aggregates these sums to update the model ${\bf w}_{k-1}$ according to the rule above.

In this work, we consider the question of how to perform this update method in a distributed and robust manner. Fix a batch (or set of points) $S_k$, which after relabeling we assume equals $\{1,\ldots, B\}$. We will denote $\nabla \ell({\bf w}_{k-1};\mathbf{x}_i)$ by $\mathbf{g}_i$. The fundamental question we consider in this work is how to compute $\sum_{i = 1}^B \mathbf{g}_i$ in a distributed and {\it adversary-resistant} manner. We present \draco{}, a framework that can compute this summation in a distributed manner, even under the presence of adversaries. 

\begin{remark}
	In contrast to previous works, our analysis and framework are applicable to any distributed algorithm which requires the sum of multiple functions. Notably, our framework can be applied to any first-order methods, including gradient descent, SVRG \cite{johnson2013accelerating}, coordinate descent, and projected or accelerated versions of these algorithms. 
	For the sake of simplicity, our discussion in the rest of the text will focus on mini-batch SGD.
\end{remark}

\paragraph{Adversarial Compute Node Model}
We consider the setting where a subset of size $s$ of the $P$ compute nodes act adversarially against the training process. The goal of an adversary can either be to completely mislead the end model, or bias it towards specific areas of the parameter space. 
A compute node is considered to be an adversarial node, if it does not return the prescribed gradient update given its allocated samples. Such a node can ship back to the PS any arbitrary update of dimension equal to that of the true gradient. Mini-batch SGD fails to converge even if there is only a single adversarial node \cite{ByzantineML_NIPS17}.

In this work, we consider the strongest possible adversaries. We assume that each adversarial node has access to infinite computational power, the entire data set, the training algorithm, and has knowledge of any defenses present in the system. Furthermore, all adversarial nodes may collaborate with each other.
\section{\draco{}: Robust Distributed Training via Algorithmic Redundancy}\label{Sec:Draco}
\label{sec:draco}

In this section we present our main results for \draco{}.
\iftoggle{tr}{The proofs are left to the appendix}{Due to space constraints, all proofs are left to the supplement}.

We generalize the scheme in Figure \ref{fig:RunningExample} to $P$ compute nodes and $B$ data samples. At each iteration of our training process, we assign the $B$ gradients to the $P$ compute nodes using a $P\times B$ {\it allocation matrix} $\mathbf{A}$. Here, $\mathbf{A}_{j,k}$ is 1 if node $j$ is assigned the $k$th gradient ${\bf g}_k$, and 0 otherwise. The support of $\mathbf{A}_{j,\cdot}$, denoted $\mathit{supp}\left(\mathbf{A}_{j,\cdot}\right)$, is the set of indices $k$ of gradients evaluated by the $j$th compute node. For simplicity, we will assume $B=P$ throughout the following.

\draco{} utilizes redundant computations, so it is worth formally defining the amount of redundancy incurred. This is captured by the following definition.

\begin{definition}$r\triangleq \frac{1}{P}  \|\mathbf{A}\|_0$ denotes the redundancy ratio. \end{definition}

In other words, the redundancy ratio is the average number of gradients assigned to each compute node.

We define a $d \times P$ matrix $\mathbf{G}$ by $\mathbf{G} \triangleq [\mathbf{g}_1,\mathbf{g}_2,\cdots, \mathbf{g}_P]$. Thus, $\mathbf{G}$ has all assigned gradients as its columns. The $j$th compute node first computes a $d\times P$ gradient matrix 
$\mathbf{Y}_j \triangleq \left( \mathbf{1}_d \mathbf{A}_{j,\cdot} \right)\odot \mathbf{G}$ using its allocated gradients. 
In particular, if the $k$th gradient $\mathbf{g}_k$ is allocated to the $j$th compute node, \ie $\mathbf{A}_{j,k} \not = 0$, then the compute node computes $\mathbf{g}_k$ as the $k$th column of $\mathbf{Y}_j$.
Otherwise, it sets the $k$-th column of $\mathbf{Y}_j$ to be $\mathbf{0}_d$.

The $j$th compute node is equipped with an encoding function $E_j$ that maps the $d\times P$ matrix $\mathbf{Y}_j$ of its assigned gradients to a single $d$-dimensional vector.
After computing its assigned gradients, the $j$th compute node sends $\mathbf{z}_j \triangleq E_j(\mathbf{Y}_j)$ to the PS.
If the $j$th node is adversarial then it instead sends $\mathbf{z}_j+\mathbf n_{j}$ to the PS, where $\mathbf n_{j}$ is an arbitrary $d$-dimensional vector.
We let $E$ be the set of local encoding functions, \ie $E = \{ E_1, E_2, \cdots, E_P\}$.

Let us define a $d\times P$ matrix $\mathbf{Z}^{\mathbf{A},E,\mathbf{G}}$ by $\mathbf{Z}^{\mathbf{A},E,\mathbf{G}} \triangleq [\mathbf{z}_1, \mathbf{z}_2,\cdots, \mathbf{z}_P]$, and a $d\times P$ matrix $\mathbf{N}$ by $\mathbf{N} \triangleq [\mathbf{n}_1, \mathbf{n}_2,\cdots, \mathbf{n}_P]$. Note that at most $s$ columns of $\mathbf{N}$ are non-zero.
Under this notation, after all updates are finished the PS receives a $d\times P$ matrix $\mathbf{R} \triangleq \mathbf{Z}^{\mathbf{A},E,\mathbf{G}} +\mathbf{N}$. The PS then computes a $d$-dimensional update gradient vector $\mathbf{u} \triangleq D(\mathbf{R})$ using a decoder function $D$.

The system in \draco{} is determined by the tuple $(\mathbf{A}, E, D)$. We decide how to assign gradients by designing $\mathbf{A}$, how each compute node should locally amalgamate its gradients by designing $E$, and how the PS should decode the output by designing $D$. The process of \draco{} is illustrated in Figure \ref{fig:DracoFramework}. 

\begin{figure}[h]
	\centering
	\includegraphics[width=0.55\columnwidth]{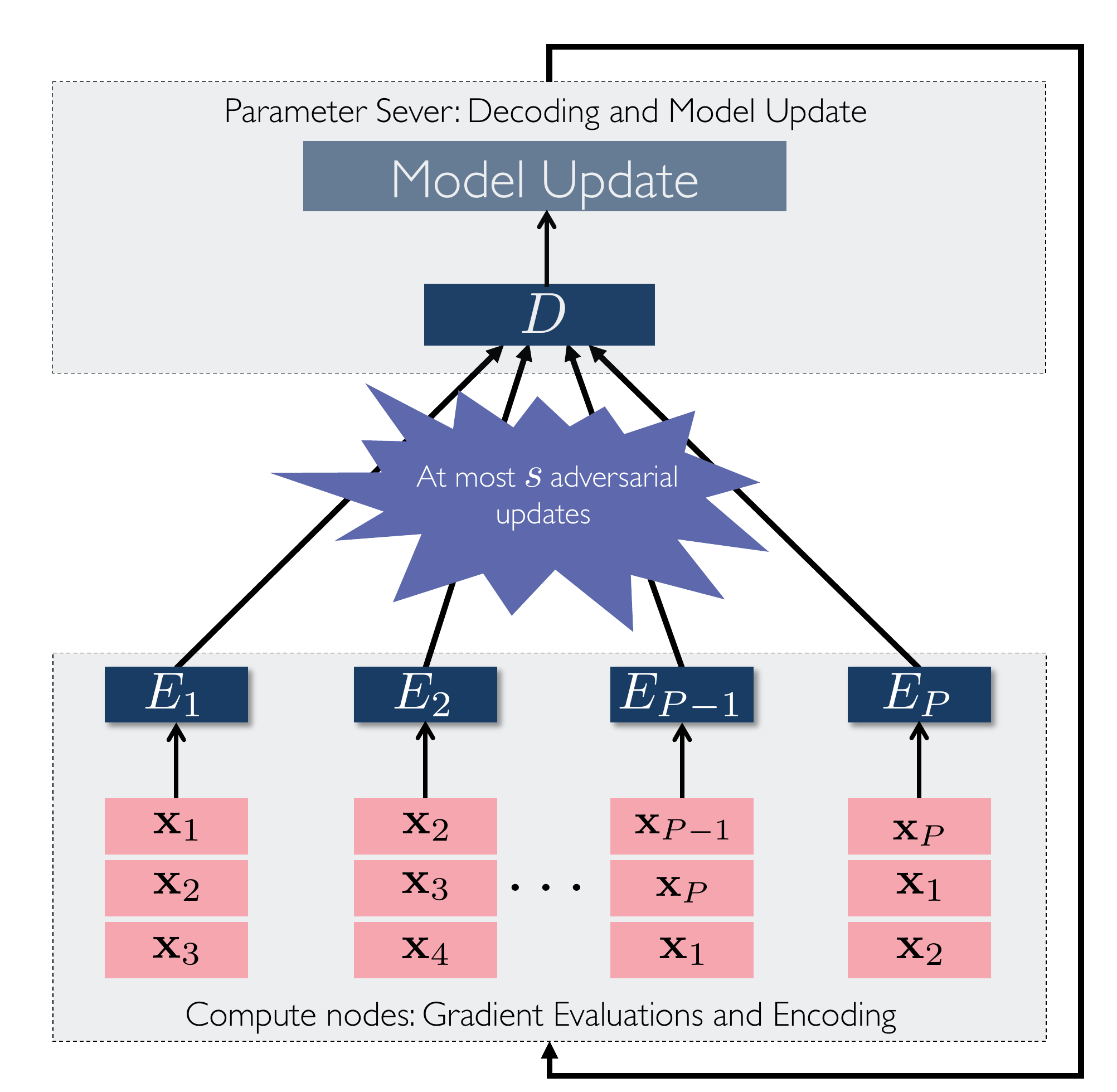}
	\caption{In \draco{}, each compute node is allocated a subset of the data set. Each compute node computes redundant gradients, encodes them via $E_i$, and sends the resulting vector to the PS. These received vectors then pass through a decoder that detects where the adversaries are and removes their effects from the updates. The output of the decoder is the true sum of the gradients. The PS applies the updates to the parameter model and we then continue to the next iteration.}
	\label{fig:DracoFramework}
\end{figure}

This framework of $(\mathbf{A}, E, D)$ encompasses both distributed SGD and the GM approach. In distributed mini-batch SGD, we assign 1 gradient to each compute node. After relabeling, we can assume that we assign $\mathbf{g}_i$ to compute node $i$. Therefore, $\mathbf{A}$ is simply the identity matrix $\mathbf{I}_P$. The matrix $\mathbf{Y}_j$ therefore contains $\mathbf{g}_j$ in column $j$ and 0 in all other columns. The local encoding function $E_j$ simply returns $\mathbf{g}_j$ by computing $E_j(\mathbf{Y}_j) = \mathbf{Y}_j\mathbf{1}_P = \mathbf{g}_j$, which it then sends to the PS. The decoding function now depends on the algorithm. For vanilla mini-batch SGD, the PS takes the average of the gradients, while in the GM approach, it takes a geometric median of the gradients.

In order to guarantee convergence, we want \draco{} to exactly recover the true sum of gradients, regardless of the behavior of the adversarial nodes. In other words, we want \draco{} to protect against {\it worst-case} adversaries.
Formally, we want the PS to always obtain the $d$-dimensional vector $\mathbf{G} \mathbf{1}_P$ via \draco{} with any $s$ adversarial nodes. 
Below is the formal definition.

\begin{definition}
	\draco{} with $(\mathbf{A}, E, D)$ can tolerate $s$ adversarial nodes, if for any $\mathbf{N} = [\mathbf{n}_1,\mathbf{n}_2,\cdots, \mathbf{n}_P]$ such that $ \left\vert{ \{j:\mathbf{n}_{j}\not=0\} }\right\vert \leq s$, we have $D(\mathbf{Z}^{\mathbf{A},E,\mathbf{G}}+\mathbf{N}) = \mathbf{G}\mathbf{1}_P$.
\end{definition}

\begin{remark}
	If we can successfully defend against $s$ adversaries, then the model update after each iteration is identical to that in the adversary-free setup. 
	This implies that any guarantees of convergence in the adversary-free case transfer to the adversarial case.
\end{remark}

\paragraph{Redundancy Bound}
We first study how much redundancy is required if we want to exactly recover the correct sum of gradients per iteration in the presence of $s$ adversaries. 
\begin{theorem}\label{Thm:RedundancyRatioBound}
	Suppose a selection of gradient allocation, encoding, and decoding mechanisms $(\mathbf{A}, E, D)$ can tolerate $s$ adversarial nodes.
	Then its redundancy ratio $r$ must satisfy $r \geq 2s+1$.
\end{theorem}

The above result is information--theoretic, meaning that regardless of how the compute node encodes and how the PS decodes, each data sample has to be replicated at least $2s+1$ times to defend against $s$ adversarial nodes. 

\begin{remark}
	Suppose that a tuple $(\mathbf{A}, E, D)$ can tolerate any $s$ adversarial nodes. By Theorem \ref{Thm:RedundancyRatioBound}, this implies that on average, each compute node encodes at least $(2s+1)$ $d$-dimensional vectors. Therefore, if the encoding has linear complexity, then each encoder requires $(2s+1)d$ operations in the worst-case. If the decoder $D$ has linear time complexity, then it requires at least $Pd$ operations in the worst case, as it needs to use the $d$-dimensional input from all $P$ compute nodes. This gives a computational cost of $O(Pd)$ in general, which is significantly less than that of the median approach in \cite{ByzantineML_NIPS17}, which requires $\mathcal{O}(P^2 (d +\log P))$ operations.
\end{remark}

\paragraph{Optimal Coding Schemes}
A natural question is, \textit{can we achieve the optimal redundancy bound with linear-time encoding and decoding}? 
More formally, can we design a tuple $(\mathbf{A},E,D)$ that has redundancy ratio $r=2s+1$ and computation complexity $\mathcal{O}((2s+1)d)$ at the compute node and $\mathcal{O}(Pd)$ at the PS?
We give a positive answer by presenting two coding approaches that match the above bounds. The encoding methods are based on the fractional repetition code and the cyclic repetition codes in \cite{Gradient_Coding,raviv2017gradient}.

\paragraph{Fractional Repetition Code}
Suppose $2s+1$ divides $P$. 
The fractional repetition code (derived from \cite{Gradient_Coding}) works as follows.
We first partition the compute nodes into $r = 2s+1$ groups.
We assign the nodes in a group to compute the same sum of gradients. Let $\hat{\mathbf{g}}$ be the desired sum of gradients per iteration. In order to decode the outputs returned by the compute nodes in the same group, the PS uses majority vote to select one value.
This guarantees that as long as fewer than half of the nodes in a group are adversarial, the majority procedure will return the correct $\hat{\mathbf{g}}$.

Formally, the repetition code $(\mathbf{A}^{\mathit{Rep}}, E^{\mathit{Rep}}, D^{Rep})$ is defined as follows. 
The assignment matrix $\mathbf{A}^{\mathit{Rep}}$ is given by
	\[
	\mathbf{A}^{\mathit{Rep}} =
	\begin{bmatrix}
	\mathbf{1}_{r \times r} & \mathbf{0}_{r \times r} &  \mathbf{0}_{r \times r} &\cdots &  \mathbf{0}_{r \times r} &  \mathbf{0}_{r \times r}\\
	\mathbf{0}_{r \times r} & \mathbf{1}_{r \times r} &  \mathbf{0}_{r \times r} &\cdots &  \mathbf{0}_{r \times r} &  \mathbf{0}_{r \times r} \\
	\vdots & \vdots & \vdots & \ddots & \vdots & \vdots\\
	\mathbf{0}_{r \times r} &  \mathbf{0}_{r \times r} & \mathbf{0}_{r \times r} & \cdots &  \mathbf{0}_{r \times r} & \mathbf{1}_{r \times r} \\
	\end{bmatrix}.
	\\
	\]
The $j$th compute node first computes all its allocated gradients $\mathbf{Y}^{Rep}_j=\left( \mathbf{1}_d \mathbf{A}^{Rep}_{j,\cdot} \right)\odot \mathbf{G}$.
Its encoder function simply takes the summation of all the allocated gradients. That is, $E_j^{\mathit{Rep}}(\mathbf{Y}^{Rep}_j) = \mathbf{Y}^{Rep}_j\mathbf{1}_P$.
It then sends $\mathbf{z}_j = E_j^{\mathit{Rep}}(\mathbf{Y}^{Rep}_j)$ to the PS.

The decoder works by first finding the majority vote of the output of each compute node that was assigned the same gradients. For instance, since the first $r$ compute nodes were assigned the same gradients, it finds the majority vote of $[\mathbf{z}_1,\ldots, \mathbf{z}_r]$. It does the same with each of the blocks of size $r$, and then takes the sum of the $P/r$ majority votes. We note that our decoder here is different compared to the one used in the straggler mitigation setup of \cite{Gradient_Coding}. Our decoder follows the concept of majority decoding similarly to \cite{von1956probabilistic,pippenger1988reliable}.

Formally, $D^{\mathit{Rep}}$ is given by $D^{\mathit{Rep}}(\mathbf{R}) = \sum_{\ell = 1}^{\frac{P}{r}} \mathit{Maj}\left(\mathbf{R}_{\cdot,\left(\ell\cdot (r-1) + 1\right): \left(\ell\cdot r\right)}\right)$, where $\mathit{Maj}(\cdot)$ denotes the majority vote function and $\mathbf{R}$ is the $d \times P $  matrix received from all compute nodes.  
While a naive implementation of majority vote scales quadratically with the number of compute nodes $P$, we instead use a streaming version of majority vote \cite{MajorityVote}, the complexity of which is linear in $P$. 

\begin{theorem}\label{Thm:RepCodeOptimality}
	Suppose $2s+1$ divides $P$. Then the repetition code $(\mathbf{A}^{\mathit{Rep}}, E^{\mathit{Rep}}, D^{Rep})$ with $r=2s+1$ can tolerate any $s$ adversaries, achieves the optimal redundancy ratio, and has linear-time encoding and decoding.
\end{theorem}  	

\paragraph{Cyclic Code}
Next we describe a cyclic code whose encoding method comes from \cite{Gradient_Coding} and is similar to that of \cite{raviv2017gradient}. We denote the cyclic code, with encoding and decoding functions, by $(\mathbf{A}^{\mathit{Cyc}}, E^{\mathit{Cyc}}, D^{Cyc})$. The cyclic code provides an alternative way to tolerate adversaries in distributed setups.
We will show that the cyclic code also achieves the optimal redundancy ratio and has linear-time encoding and decoding. Another difference compared to the repetition code is that in the cyclic code, the compute nodes will compute and transmit complex vectors, and the decoding function will take as input these complex vectors.

To better understand the cyclic code, imagine that all $P$ gradients we wish to compute are arranged in a circle. 
Since there are $P$ starting positions, there are $P$ possible ways to pick a sequence consisting of $2s+1$ clock-wise consecutive gradients in the circle. Assigning each sequence of gradients to each compute node leads to redundancy ratio $r=2s+1$. 
The allocation matrix for the cyclic code is $\mathbf{A}^{\mathit{Cyc}}$, where the $i$ row contains $r=2s+1$ consecutive ones, between position $(i-1)r+1$ to $i\cdot r$ modulo $B$.

In the cyclic code, each compute node computes a linear combination of its assigned gradients. This can be viewed as a generalization of the repetition code's encoder.
Formally, we  construct some ${P\times P}$ matrix $\mathbf{W}$ such that $\forall j, \ell, \mathbf{A}^{Cyc}_{j,\ell} \neq 0$ implies  $\mathbf{W}_{j,\ell} = 0 $.  
Let $\mathbf{Y}^{Cyc}_j=\left( \mathbf{1}_d \mathbf{A}^{Cyc}_{j,\cdot} \right)\odot \mathbf{G}$ denote the gradients computed at compute node $j$.
The local encoding function $E_j^{\mathit{Cyc}}$ is defined by
$E_j^{\mathit{Cyc}}(\mathbf{Y}_j^{Cyc})
= \mathbf{G} \mathbf{W}_{\cdot,j}.$
After performing this local encoding, the $j$th compute node then sends $\mathbf{z}^{Cyc}_j \triangleq E_j^{\mathit{Cyc}}(\mathbf{Y}_j^{Cyc})$ to the PS.
Let $\mathbf{Z}^{\mathbf{A}^{Cyc},E^{Cyc},\mathbf{G}} \triangleq [\mathbf{z}^{Cyc}_1, \mathbf{z}^{Cyc}_2, \cdots, \mathbf{z}^{Cyc}_P]$. Then one can verify from the definition of $E_j^{\mathit{Cyc}}$ that $\mathbf{Z}^{\mathbf{A}^{Cyc},E^{Cyc},\mathbf{G}} = \mathbf{G} \mathbf{W}$.
The received matrix at the PS now becomes $\mathbf{R}^{Cyc} = \mathbf{Z}^{\mathbf{A}^{Cyc},E^{Cyc},\mathbf{G}} + \mathbf{N} = \mathbf{G} \mathbf{W} +\mathbf{N}$.

In order to decode, the PS needs to detect which compute nodes are adversarial and recover the correct gradient summation from the non-adversarial nodes. Methods to do the latter alone in the presence of straggler nodes was presented in \cite{Gradient_Coding} and \cite{raviv2017gradient}. Suppose there is a function $\phi(\cdot)$ that can compute the adversarial node index set $V$. We will later construct $\phi$ explicitly.
Let $U$ be the index set of the non-adversarial nodes. 
Suppose that the span of $\mathbf{W}_{\cdot,U}$ contains $\mathbf{1}_P$.
Thus, we can obtain a vector $\mathbf{b}$ by solving $\mathbf{W}_{\cdot,U} \mathbf{b}  = \mathbf{1}_P$.
Finally, since $U$ is the index set of non-adversarial nodes, for any $j \in U$, we must have $\mathbf{n}_{j} = \mathbf{0}$. 
Thus, we can use $\mathbf{R}^{Cyc}_{\cdot, U} \mathbf{b} =  (\mathbf{G}\mathbf{W}+\mathbf{N})_{\cdot,U} \mathbf{b} = \mathbf{G}\mathbf{W}_{\cdot,U} \mathbf{b} = \mathbf{G} \mathbf{1}_P$.
The decoder function is given formally in Algorithm \ref{Alg:DecoderFunction.}.
\begin{algorithm}[htbp]
		\SetKwInOut{Input}{Input}
		\SetKwInOut{Output}{Output}
		\Input{Received $d \times P$ matrix $\mathbf{R}^{Cyc}$}
		\Output{Desired gradient summation $\mathbf{u}^{Cyc}$}
		$V = \phi(\mathbf{R})$ // Locate the adversarial node indexes.\\
		$U = \{1,2,\cdots,P\} - V$. // Non-adversarial node indexes \\
		Find $\mathbf{b}$ by solving $ \mathbf{W}_{\cdot,U} \mathbf{b} = \mathbf{1}_P$\\
		Compute and return $\mathbf{u}^{Cyc} = \mathbf{R}_{\cdot, U} \mathbf{b} $	\caption{Decoder Function $D^{Cyc}$.}
		\label{Alg:DecoderFunction.}
\end{algorithm}

To make this approach work, we need to design a matrix $\mathbf{W}$ and the index location function $\phi(\cdot)$ such that (i) For all $j,k$, $\mathbf{A}_{j,k} = 0 \implies \mathbf{W}_{j,k} = 0$ and the span of $\mathbf{W}_{\cdot,U}$ contains $\mathbf{1}_P$, and (ii) $\phi(\cdot)$ can locate the adversarial nodes.

Let us first construct $\mathbf{W}$.
Let $\mathbf{C}$ be a $P\times P$ inverse discrete Fourier transformation (IDFT) matrix, i.e.,
$$\mathbf{C}_{jk} = \frac{1}{\sqrt{P}}\exp\left({\frac{2 \pi i }{P} (j-1)(k-1)}\right),~~j,k=1,2,\cdots, P.$$
Let $\mathbf{C}_L$ be the first $P-2s$ rows of $\mathbf{C}$ and $\mathbf{C}_R$ be the last $2s$ rows.
Let $\mathbf{\alpha}_{j}$ be the set of row indices of the zero entries in $\mathbf{A}^{Cyc}_{\cdot,j}$, \ie $\mathbf{\alpha}_{j} = \{k: \mathbf{A}_{j,k}^{Cyc} = 0\}$. 
Note that $\mathbf{C}_L$ is a $(P-2s)\times P$ Vandermonde matrix and thus  any $P-2s$ columns of it are linearly independent. 
Since $|\alpha_j| = P-2s-1$, we can obtain a $P-2s-1$-dimensional vector $\mathbf{q}_{j}$ uniquely by solving $\mathbf{0} = \begin{matrix} \begin{bmatrix}
\mathbf{q}_j & 1
\end{bmatrix}\end{matrix} \cdot  \left[\mathbf{C}_{L}\right]_{\cdot,\alpha_j} $.
Construct a $P\times (P-2s-1)$ matrix $\mathbf{Q} \triangleq \begin{matrix} \begin{bmatrix}
\mathbf{q}_1 & \mathbf{q}_2 & \cdots & \mathbf{q}_P
\end{bmatrix}\end{matrix}$ and a $P\times P$ matrix $\mathbf{W} \triangleq \begin{matrix} \begin{bmatrix}
\mathbf{Q} & \mathbf{1}_P
\end{bmatrix}\end{matrix} \cdot  \mathbf{C}_L$.
One can verify that (i) each row of $\mathbf{W}$ has the same support as the allocation matrix $\mathbf{A}^{Cyc}$ and (ii) the span of any $P-2s+1$ columns of $\mathbf{W}$ contains $\mathbf{1}_P$, summarized as follows.
\begin{lemma}\label{Thm:CycCodeColumnSpan}
	For all $j,k$, $\mathbf{A}_{j,k} = 0 \Rightarrow \mathbf{W}_{j,k} = 0$.
	For any index set $U$ such that $|U|\geq P-(2s+1)$, the column span of $\mathbf{W}_{\cdot,U}$ contains $\mathbf{1}_P$.
\end{lemma}

The $\phi(\cdot)$ function works as follows.
Given the $d\times P$ matrix $\mathbf{R}^{Cyc}$ received from the compute nodes, we first generate a $1\times d$ random vector $ \mathbf{f} \sim \mathcal{N}(\mathbf{1}_{1\times d},\,\mathbf{I}_{d})$,
and then compute $[h_{P-2s}, h_{P-2s-1},\cdots, h_{P-1}] \triangleq \mathbf{f} \mathbf{R} \mathbf{C}_{R}^\dag$\footnote{$\dag$ denotes transpose conjugate.}.
We then obtain a vector $\mathbf{\beta} = [\beta_0,\beta_1,\cdots, \beta_{s-1}]^T$ by solving 
	\begin{align*}
	\begin{matrix}
	\begin{bmatrix}
	h_{P-s-1}      & h_{P-s} & \dots & h_{P-2} \\
	h_{P-s-2}       & h_{P-s-1} & \dots & h_{P-3} \\
	\hdots & \hdots & \ddots &\vdots \\
	h_{P-2s}  & h_{P-s+1} & \dots & h_{P-s+1}
	\end{bmatrix}
	\begin{bmatrix}
	\beta_{0}\\
	\beta_{1}\\
	\vdots\\
	\beta_{s-1}
	\end{bmatrix}
	=
	\begin{bmatrix}
	h_{P-1}\\
	h_{P-2}\\
	\vdots\\
	h_{P-s}
	\end{bmatrix}
	\end{matrix}.
	\end{align*}We then compute $h_{\ell} = \sum_{u=0}^{s-1} \beta_{u} h_{\ell+u-s}$, where $\ell = 0,1, \cdots,P-2s-1$ and $h_{\ell} = h_{P+\ell}$.
Once the vector $\mathbf{h} \triangleq [h_0,h_1,\cdots, h_{P-1}]$ is obtained, we can compute the IDFT of $\mathbf{h}$, denoted by $\mathbf{t}\triangleq[t_0,t_1,\cdots, t_{P-1}]$.
The returned index set $V=\{j:t_{j+1}\not=0\}$.
The following lemma shows the correctness of $\phi(\cdot)$.

\begin{lemma}\label{Thm:CycCodeDetection}
	Suppose $\mathbf{N} = [\mathbf{n}_1,\mathbf{n}_2,\cdots, \mathbf{n}_P]$ satisfies $\left\vert{\{j:\|\mathbf{n}_{j}\|_0 \not=0\}}\right\vert \leq s$. 
	Then $\phi(\mathbf{R}^{Cyc}) = \phi(\mathbf{G}\mathbf{W} +\mathbf{N}) = \{j:\|\mathbf{n}_{j}\|_0 \not=0\}$ with probability 1. 
\end{lemma}

Finally we can show that the cyclic code can tolerate any $s$ adversaries and also achieves redundancy ratio and has linear-time encoding and decoding.

\begin{theorem}\label{Thm:CycCodeOptimality}
	The cyclic code  $(\mathbf{A}^{\mathit{Cyc}}, E^{\mathit{Cyc}}, D^{Cyc})$ can tolerate any $s$ adversaries with probability 1 and achieves the redundancy ratio lower bound. For $d\gg P$, its encoding and decoding achieve linear-time computational complexity.
\end{theorem}  	

Note that the cyclic code requires transmitting complex vectors $\mathbf{G}\mathbf{W}$ which potentially doubles the bandwidth requirement.
To handle this problem, one can transform the original real gradient $\mathbf{G} \in \mathbb{R}^{d\times P}$ into a complex gradient $\hat{\mathbf{G}} \in \mathbb{C}^{\ceil{d/2} \times P}$ by letting its $i$th component have real part $\mathbf{G}_i$ and complex part $\mathbf{G}_{\ceil{d/2}+i}$. 
Then the compute nodes only need to send $\mathbf{\hat{G}}\mathbf{W}$.
Once the PS recovers $\mathbf{\hat{u}}^{Cyc} \triangleq \hat{\mathbf{G}} \mathbf{1}_P$, it can simply sum the real and imaginary parts to form the true gradient summation, \ie $\mathbf{u}^{Cyc} = \operatorname{Re}(\mathbf{\hat{u}}^{Cyc}) + \operatorname{Im}(\mathbf{\hat{u}}^{Cyc}) = \mathbf{G} \mathbf{1}_P $.
\section{Experiments}\label{Sec:Experiment}
In this section we present an empirical study of \draco{} and compare it to the median-based approach in  \cite{ByzantineML_SIGMETRICS18} under different adversarial models and real distributed environments. 
The main findings are as follows: 1) For the same training accuracy, \draco{} is up to orders of magnitude faster compared to the GM-based approach; 2) In some instances, the GM approach \cite{ ByzantineML_SIGMETRICS18} does not converge, while \draco{} converges in all of our experiments, regardless of which dataset, machine learning model, and adversary attack model we use; 3) Although \draco{} is faster than GM-based approaches, its runtime can sometimes scale linearly with the number of adversaries due to the algorithmic redundancy needed to defend against adversaries.

\paragraph{Implementation and Setup}
We compare vanilla mini-batch SGD to both \draco{}-based mini-batch SGD and GM-based mini-batch SGD \cite{ ByzantineML_SIGMETRICS18}.
In mini-batch SGD, there is no data replication and each compute node only computes gradients sampled from its partition of the data. 
The PS then averages all received gradients and updates the model.
In GM-based mini-batch SGD, the PS uses the geometric median instead of average to update the model.
We have implemented all of these in PyTorch \cite{paszke2017automatic} with MPI4py \cite{MPI4PY} deployed on the m4.2/4/10xlarge instances in Amazon EC2 \footnote{\url{https://github.com/hwang595/Draco}}.
We conduct our experiments on various adversary attack models, datasets,  learning problems and neural network models.

\paragraph{Adversarial Attack Models}
We consider two adversarial models.
First, we consider  the ``reversed gradient" adversary, where adversarial nodes that were supposed to send ${\bf g}$ to the PS instead send  $-c {\bf g}$, for some $c>0$.
Next, we consider a ``constant adversary" attack, where adversarial nodes always send a constant multiple $\kappa$ of the all-ones vector to the PS with dimension equal to that of the true gradient.
In our experiments, we set $c=100$ for the reverse gradient adversary, and $\kappa =  -100$ for the constant adversary.
In either setup, at each iteration, $s$ nodes are randomly selected to act as adversaries.
 
\paragraph{End-to-end Convergence Performance}
\begin{figure*}[ht]
	\vspace{-0.2cm}
	\centering
	\includegraphics[width=1\linewidth]{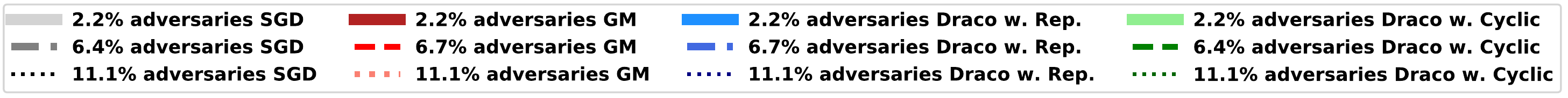}\\
	\subfigure[MNIST,FC,Rev Grad]{\includegraphics[width=0.24\linewidth]{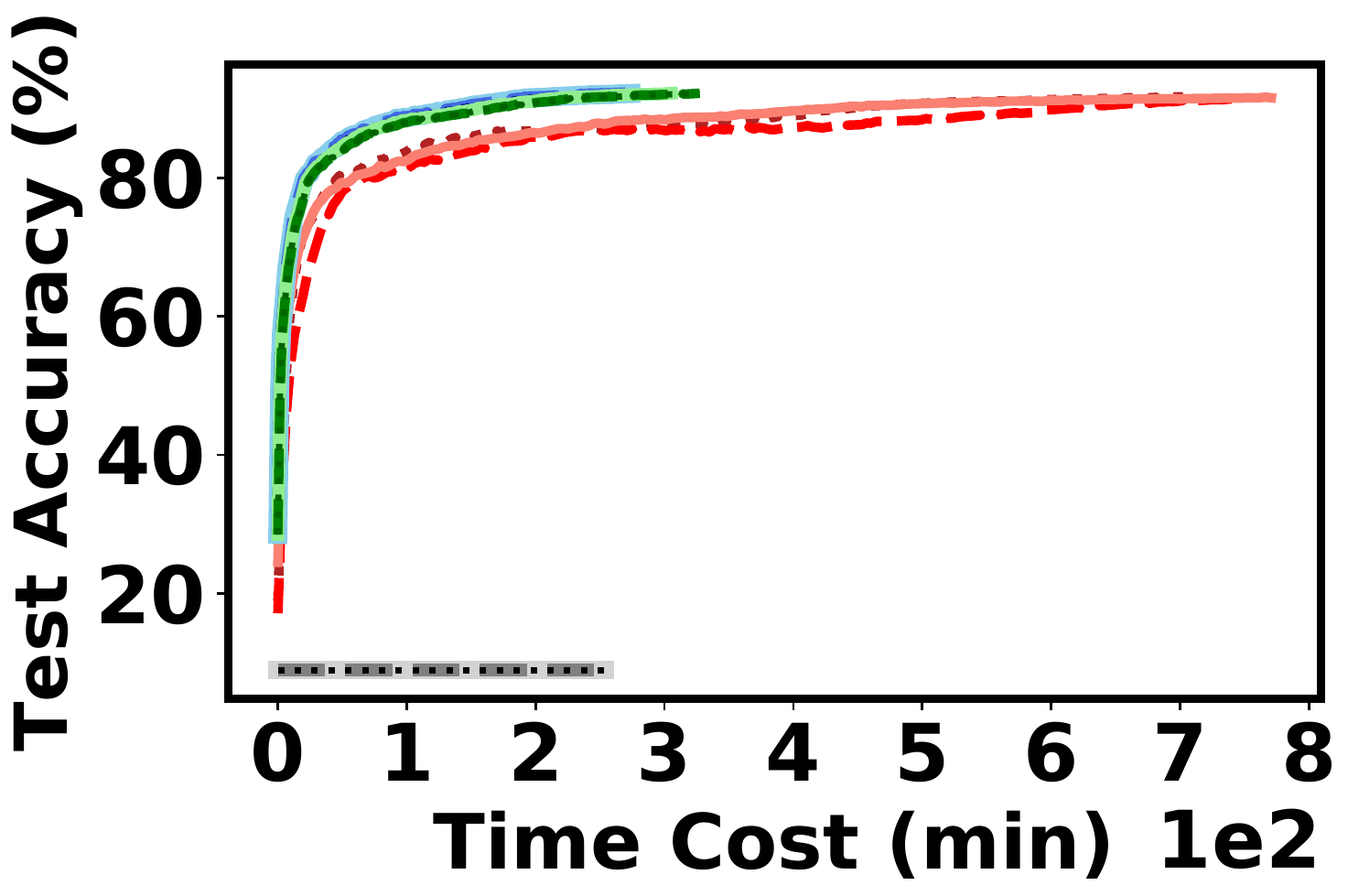}}
	\subfigure[MNIST,LeNet,Rev Grad]{\includegraphics[width=0.24\linewidth]{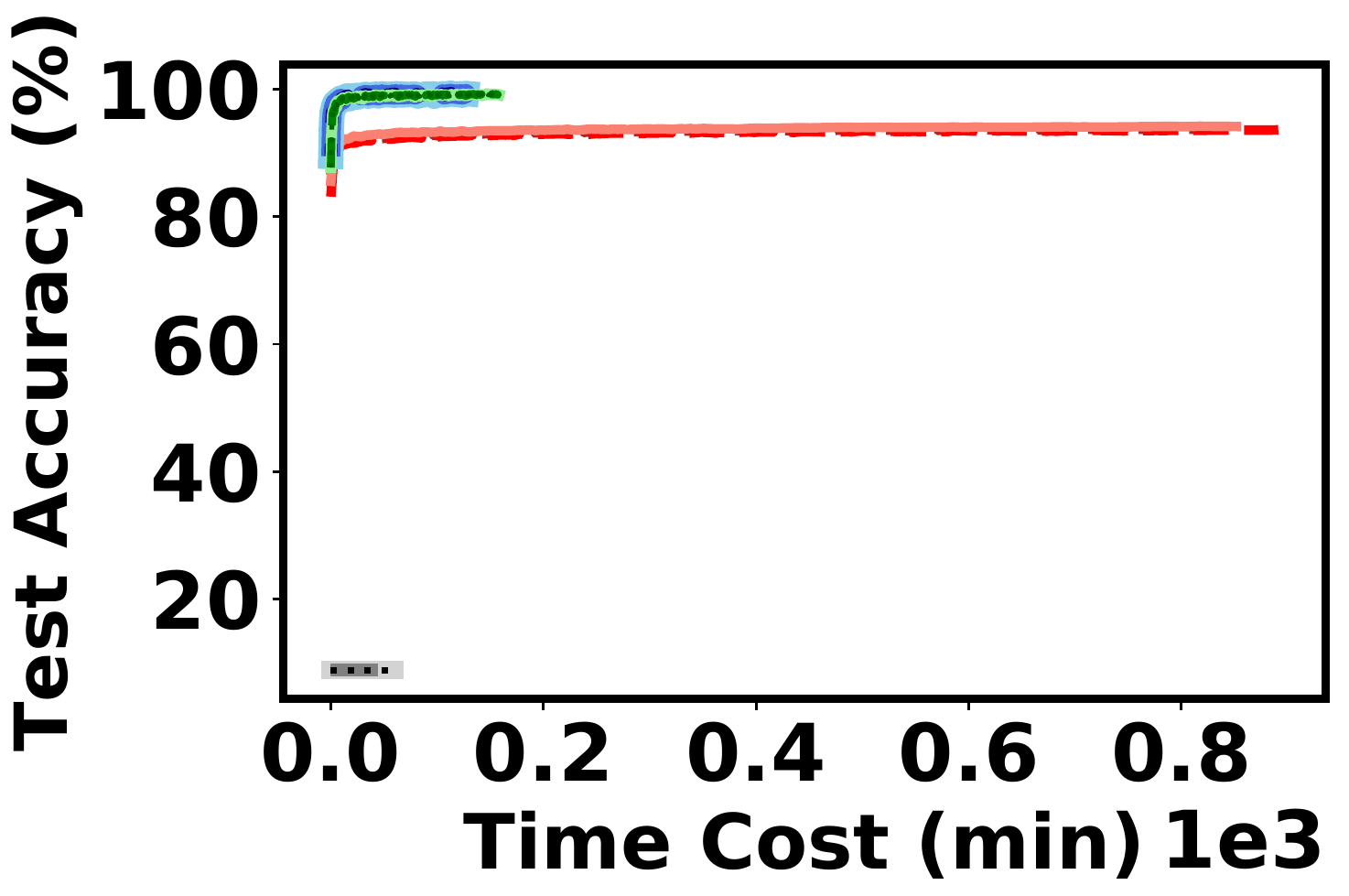}}
	\subfigure[CIFAR10,ResNet18,Rev Grad]{\includegraphics[width=0.24\linewidth]{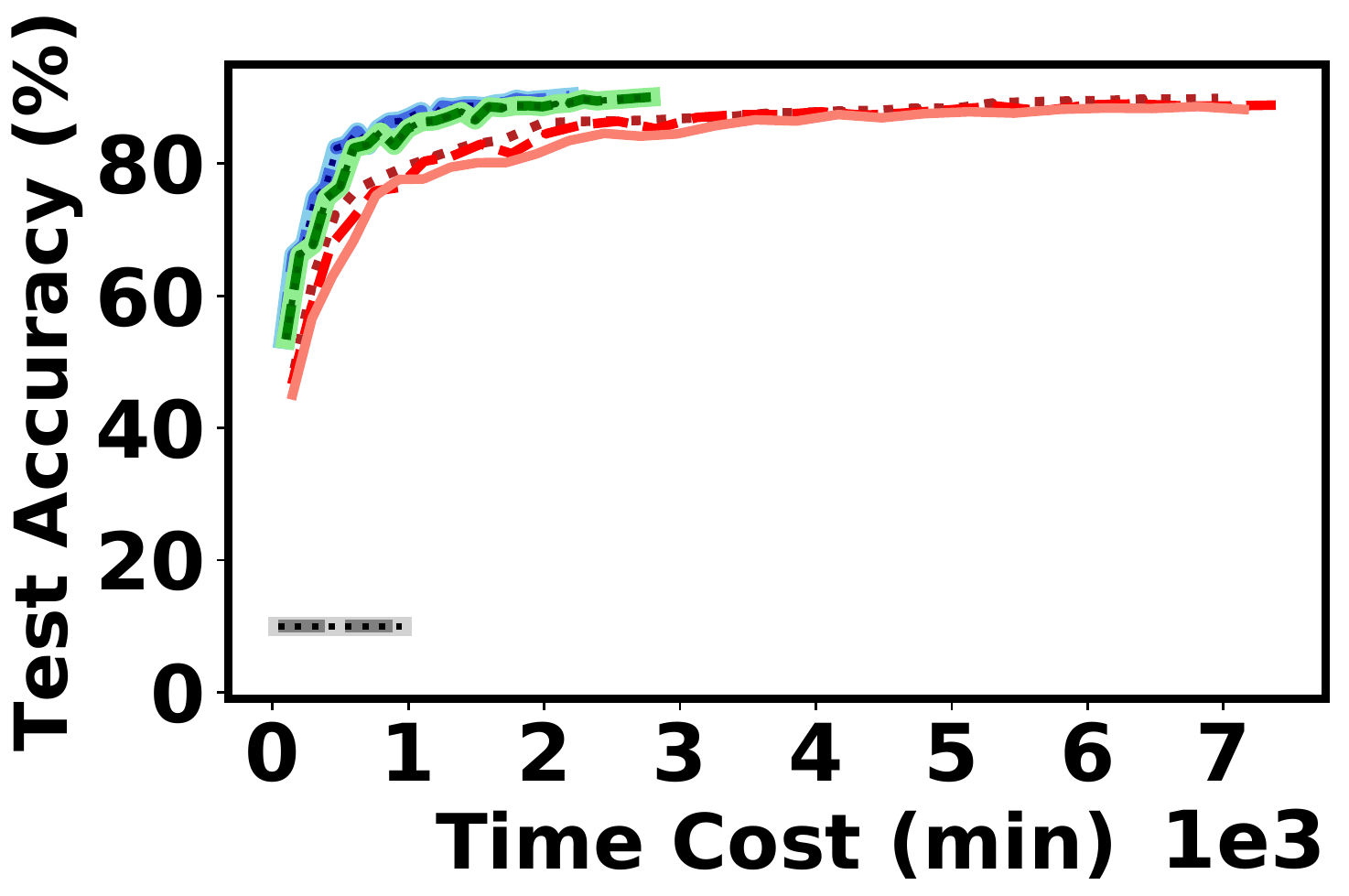}}
	\subfigure[MR,CRN,Rev Grad]{\includegraphics[width=0.24\linewidth]{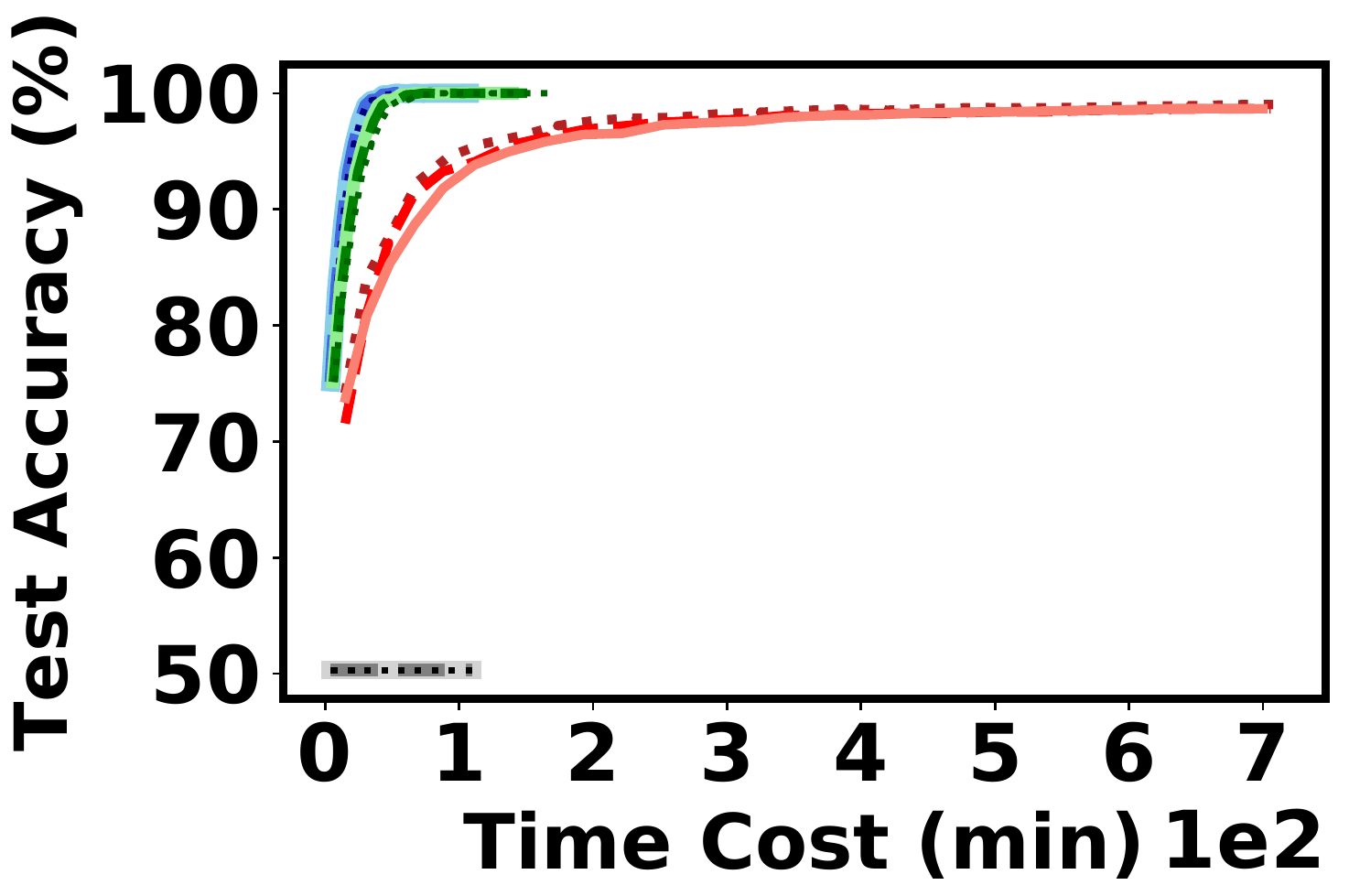}}
	\\
	\subfigure[ MNIST,FC,Const]{\includegraphics[width=0.24\linewidth]{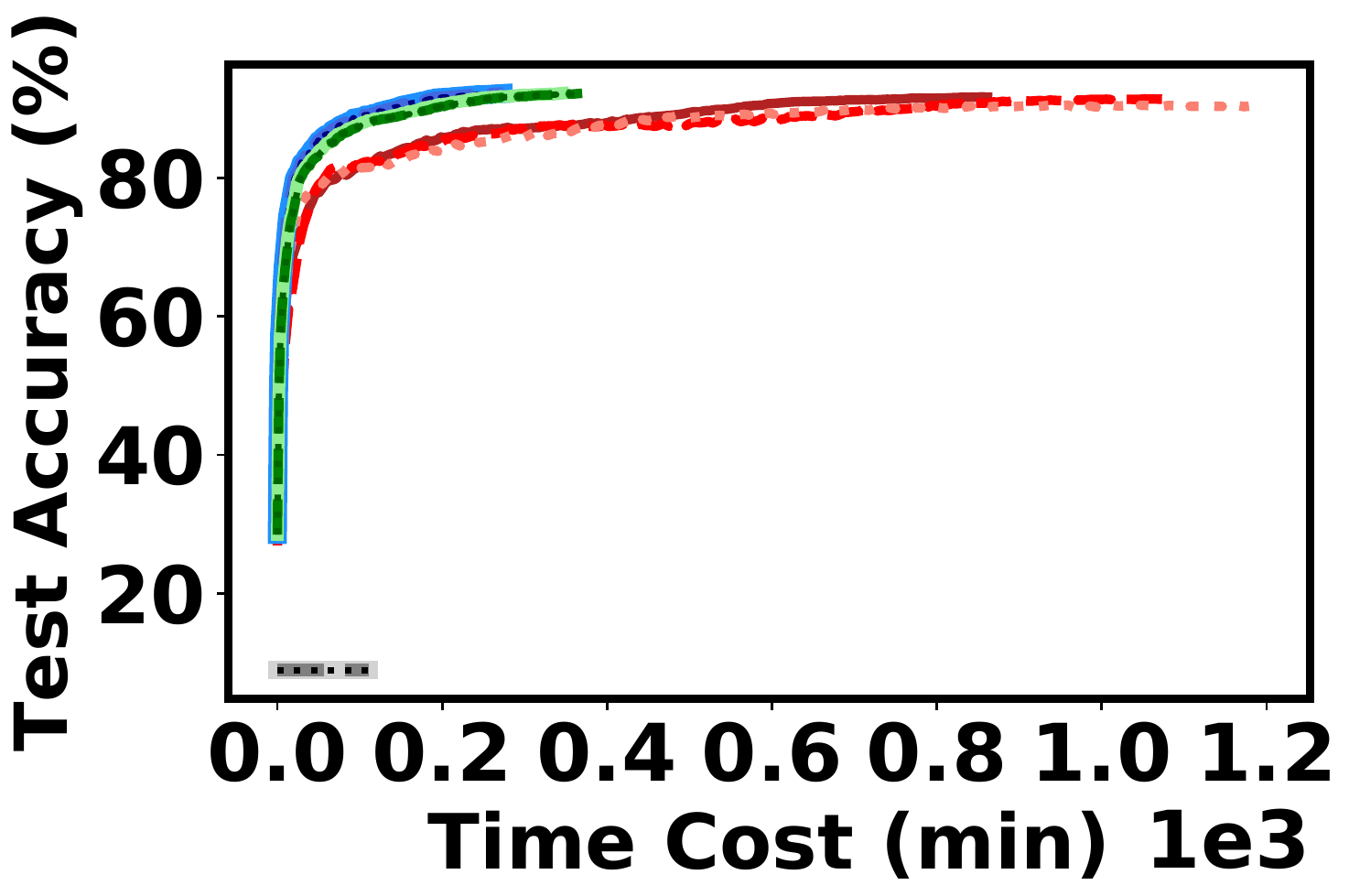}}
	\subfigure[ MNIST,LeNet,Const]{\includegraphics[width=0.24\linewidth]{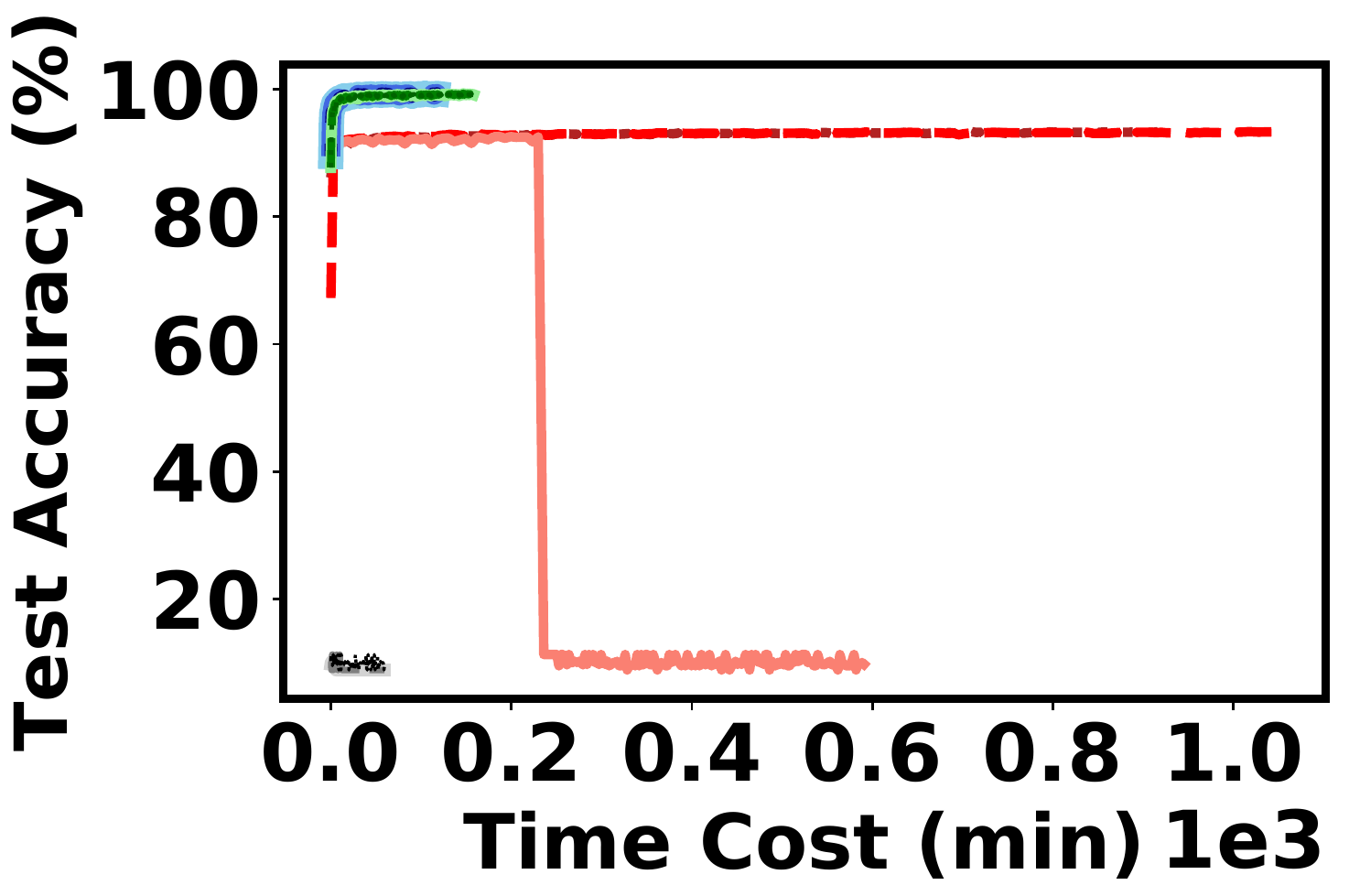}}
	\subfigure[ CIFAR10,ResNet18,Const]{\includegraphics[width=0.24\linewidth]{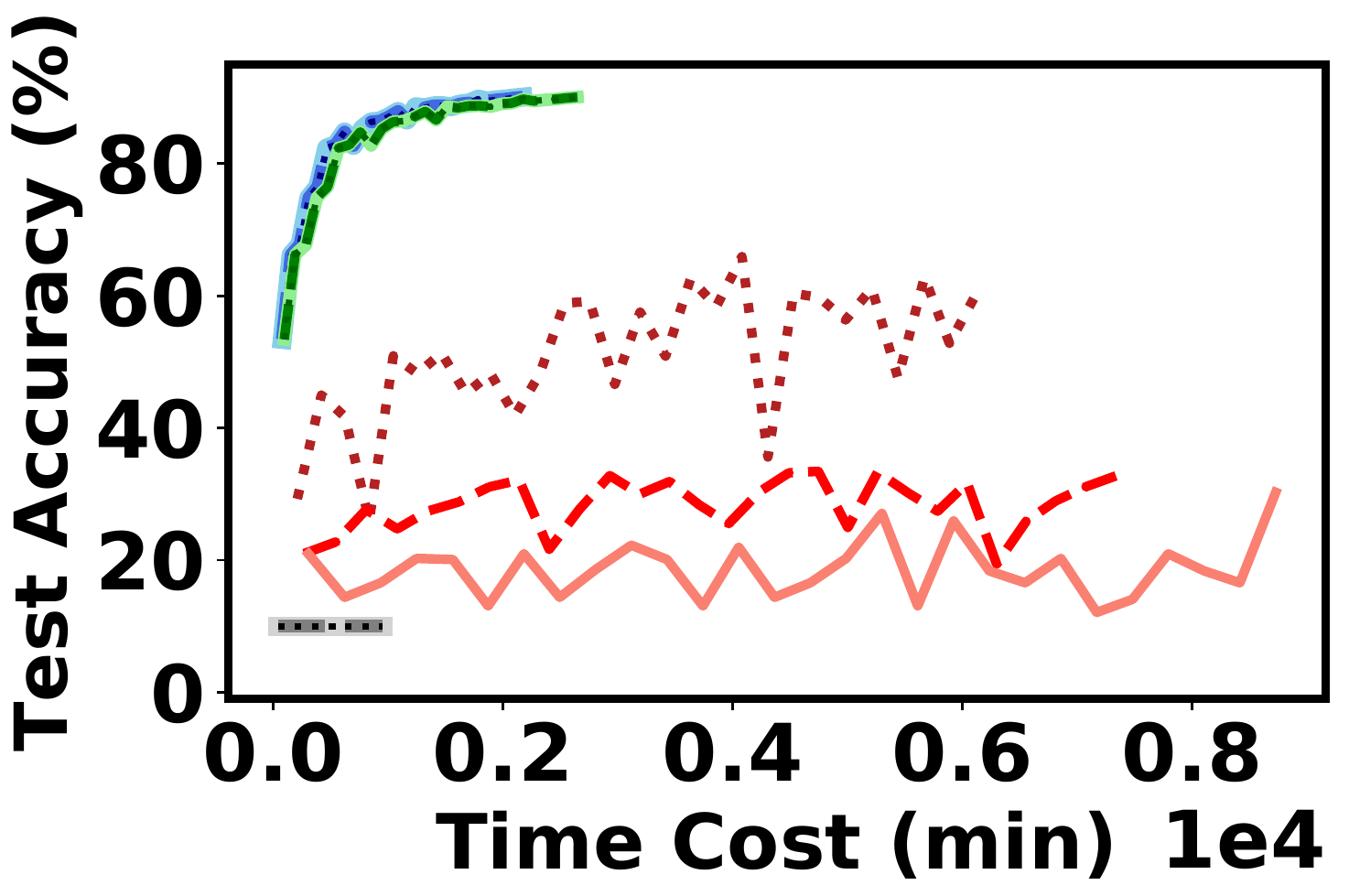}}
	\subfigure[MR,CRN,Const]{\includegraphics[width=0.24\linewidth]{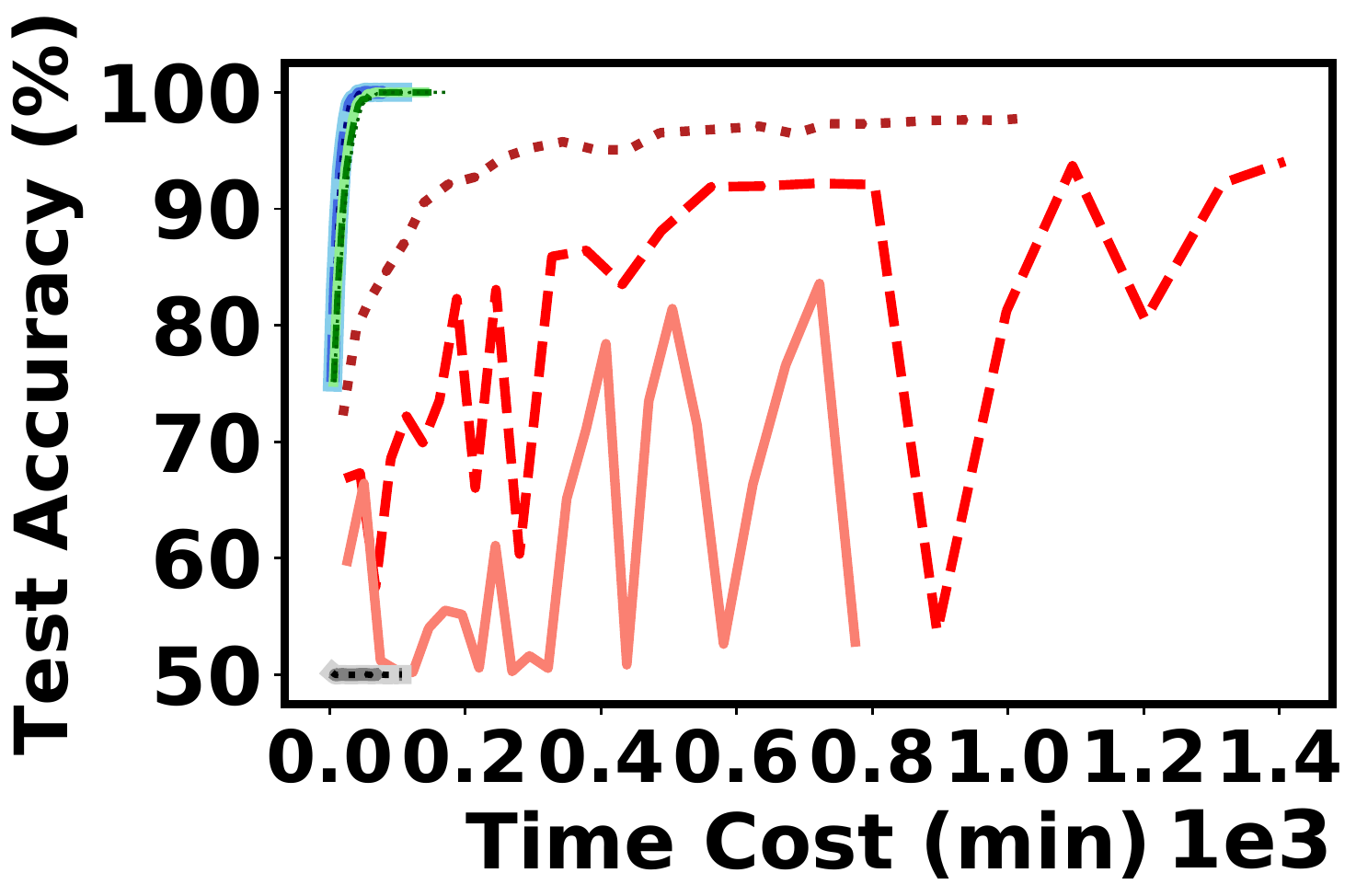}}	
	\caption{Convergence rates of \draco{}, GM, and vanilla mini-batch SGD, on (a) MNIST on FC, (b) MNIST on LeNet, (c) CIFAR10 on ResNet-18, and (d) MR on CRN, all with reverse gradient adversaries; (e) MNIST on FC, (f) MNIST on LeNet,  (g) CIFAR10 on ResNet-18, and (h) MR on CRN, all with constant adversaries.}
	\label{fig:ConvergeModel}
\end{figure*}
We first evaluate the end-to-end convergence performance of \draco{}, using both the repetition and cyclic codes, and compare it to ordinary mini-batch SGD as well as the GM approach.
\begin{table}[ht]
	\centering
	\caption{The datasets used, their associated learning models and corresponding parameters.}
	\vspace{0.3 cm}
	\begin{tabular}{|c|c|c|c|}
		\hline Dataset
		& MNIST & Cifar10  & MR \bigstrut\\
		\hline
		\# data points & 70,000 & 60,000 & 10,662 \bigstrut\\
		\hline
		Model & FC/LeNet & ResNet-18 & CRN \bigstrut\\
		\hline
		\# Classes & 10 & 10 & 2 \bigstrut\\
		\hline
		\# Parameters & 1,033k / 431k & 1,1173k & 154k \bigstrut\\
		\hline
		Optimizer & SGD & SGD & Adam \bigstrut\\
		\hline
		Learning Rate & 0.01 / 0.01 & 0.1 & 0.001 \bigstrut\\
		\hline
		Batch Size & 720 / 720 & 180 & 180 \bigstrut\\
		\hline
	\end{tabular}%
	\label{Tab:DataStat}%
\end{table}
The datasets and their associated learning models are summarized in Table \ref{Tab:DataStat}. 
We use fully connected (FC) neural networks and LeNet \cite{LeNet} for MNIST, ResNet-18 \cite{ResNet} for Cifar 10 \cite{Cifar10}, and CNN-rand-non-static (CRN) model in \cite{DBLP:conf/emnlp/Kim14} for Movie Review (MR) \cite{MR_Data}.

The experiments were run on a cluster of 45 compute nodes instantiated on m4.2xlarge instances. 
At each iteration, we randomly select $s=1,3,5$ (2.2\%, 6.7\%, 11.1\% of all compute nodes) nodes as adversaries.
All three methods are trained for 10,000 distributed iterations.
Figure \ref{fig:ConvergeModel} shows how the testing accuracy varies with training time.
Tables \ref{tab:SpeedupMNIST}, \ref{tab:SpeedupCIFAR10}, and \ref{tab:SpeedupCRM} give a detailed account of the speedups of \draco{} compared to the GM approach, where we run both systems until they achieve the same designated testing accuracy.

\begin{table}[ht]
	\centering
	\caption{Speedups (\ie $X$ times faster) of \draco{} (Repetition/Cyclic Codes) over GM when using a fully-connected neural network on the MNIST dataset. We run both methods until they reach the same specified testing accuracy. In the table `const' and `rev grad' refer to the two types of adversarial updates.}
	\vspace{0.3 cm}

	\begin{tabular}{|c|c|c|c|c|}
		\hline Test Accuracy
		& 80\% & 85\%  & 88\%  & 90\%  \bigstrut\\
		\hline
				\hline
		2.2\% const & \textbf{3.4}/\textbf{2.7}  & \textbf{3.5}/\textbf{2.8} & \textbf{4.8}/\textbf{3.9} & \textbf{4.1}/\textbf{3.1} \bigstrut\\
		\hline
		6.7\% const & \textbf{2.7}/\textbf{2.0} & \textbf{4.1}/\textbf{3.1} & \textbf{6.0}/\textbf{4.6} & \textbf{5.6}/\textbf{4.1} \bigstrut\\
		\hline
		11.1\% const & \textbf{2.9}/\textbf{2.2} & \textbf{4.8}/\textbf{3.7} & \textbf{6.1}/\textbf{4.7} & \textbf{5.3}/\textbf{3.8} \bigstrut\\
		\hline
		2.2\% rev grad & \textbf{2.2}/\textbf{1.9} & \textbf{2.4}/\textbf{2.2} & \textbf{4.1}/\textbf{3.7} & \textbf{3.2}/\textbf{2.9} \bigstrut\\
		\hline
		6.7\% rev grad & \textbf{3.1}/\textbf{2.5} & \textbf{3.3}/\textbf{3.1} & \textbf{5.5}/\textbf{4.8} & \textbf{4.5}/\textbf{3.7} \bigstrut\\
		\hline
		11.1\% rev grad & \textbf{2.7}/\textbf{2.3} & \textbf{3.0}/\textbf{2.6} & \textbf{3.1}/\textbf{2.7} & \textbf{3.1}/\textbf{2.6} \bigstrut\\
		\hline
	\end{tabular}%
	\label{tab:SpeedupMNIST}%
\end{table}%
First, as expected, ordinary mini-batch may not converge even if there is only one adversary.
Second, under the \textit{reverse gradient adversary} model, \draco{} converges several times faster than the GM approach, using both the repetition and cyclic codes.
In fact, as shown in the speedup tables, both the repetition and the cyclic code versions of \draco{} achieve up to more than an order of magnitude speedup compared to the GM approach.
We suspect that this is because the computation of the GM is extremely expensive compared to the encoding and decoding overhead of \draco{}.

\begin{table}[htbp]
	\centering
	\caption{Speedups of \draco{} (with both repetition and cyclic codes) over GM when using ResNet-18 on Cifar10. We run both methods until they reach the same specified testing accuracy. Here $\infty$ means that the GM approach failed to converge to the same accuracy reached by \draco{}. }

	\begin{tabular}{|c|c|c|c|c|}
		\hline Test Accuracy
		& 80\% & 85\%  & 88\%  & 90\%  \bigstrut\\
		\hline
						\hline
		2.2\% rev grad & \textbf{2.6}/\textbf{2.0} & \textbf{3.3}/\textbf{2.6} & \textbf{4.2}/\textbf{3.3} & \textbf{$\infty$}/\textbf{$\infty$} \bigstrut\\
		\hline
		6.7\% rev grad & \textbf{2.8}/\textbf{2.2} & \textbf{3.4}/\textbf{2.7} & \textbf{4.3}/\textbf{3.4} & \textbf{$\infty$}/\textbf{$\infty$} \bigstrut\\
		\hline
		11.1\% rev grad & \textbf{4.1}/\textbf{3.3} & \textbf{4.2}/\textbf{3.2} & \textbf{5.5}/\textbf{4.4} & \textbf{$\infty$}/\textbf{$\infty$} \bigstrut\\
		\hline
	\end{tabular}%

	\label{tab:SpeedupCIFAR10}%
\end{table}%

\begin{table}[htbp]	
	\centering
	\caption{Speedups of \draco{} (with both repetition and cyclic codes) over GM when using CRM on MR. We run both methods until they reach the same specified testing accuracy.}

	\begin{tabular}{|c|c|c|c|c|}
		\hline Test Accuracy
		& 95\% & 96\%  & 98\%  & 98.5\%  \bigstrut\\
		\hline
						\hline
		2.2\% rev grad & \textbf{5.4}/\textbf{4.2} & \textbf{5.6}/\textbf{4.3} & \textbf{9.7}/\textbf{7.4} & \textbf{12}/\textbf{9.0} \bigstrut\\
		\hline
		6.7\% rev grad & \textbf{6.4}/\textbf{4.5} & \textbf{6.3}/\textbf{4.5} & \textbf{11}/\textbf{8.1} & \textbf{19}/\textbf{13} \bigstrut\\
		\hline
		11.1\% rev grad & \textbf{7.5}/\textbf{4.7} & \textbf{7.4}/\textbf{4.6} & \textbf{12}/\textbf{8} & \textbf{19}/\textbf{12} \bigstrut\\
		\hline
	\end{tabular}
	\label{tab:SpeedupCRM}%
\end{table}%

Under the \textit{constant adversary} model, the GM approach sometimes failed to converge while \draco{} still converged in all of our experiments.
This reflects our theory, which shows that \draco{} always returns a model identical to the model trained by the ordinary algorithm in an adversary-free environment.
One reason why the GM approach may fail to converge is that by using the geometric median, it is actually losing information about a subset of the gradients.
Under the constant adversary model, the PS effectively gains no information about the gradients computed by the adversarial nodes, and cannot recover the desired optimal model.

Another reason that GM may not converge may be because theoretical convergence guarantees of GM require certain assumptions on the underlying models, such as convexity. Since neural networks are generally non-convex, we have no guarantees that GM converges in these settings.
It is worth noting that GM may also not converge if we were to use an algorithm such as L-BFGS or accelerated gradient descent, as the choice of algorithm is separate from the underlying properties of the neural network models.
Nevertheless, \draco{} still converges for such algorithms.

\paragraph{Per iteration cost of \draco{}}
\begin{figure*}[ht]
	\centering
	\includegraphics[width=\linewidth]{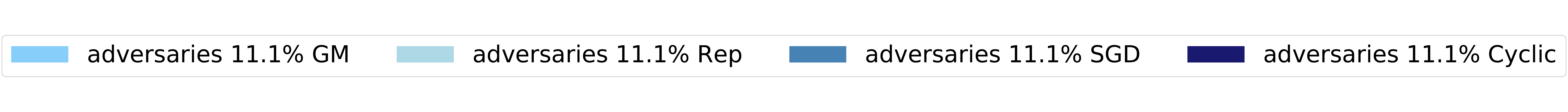}\\
	\subfigure[ResNet-152, Rev Grad]{\includegraphics[width=0.3\linewidth]{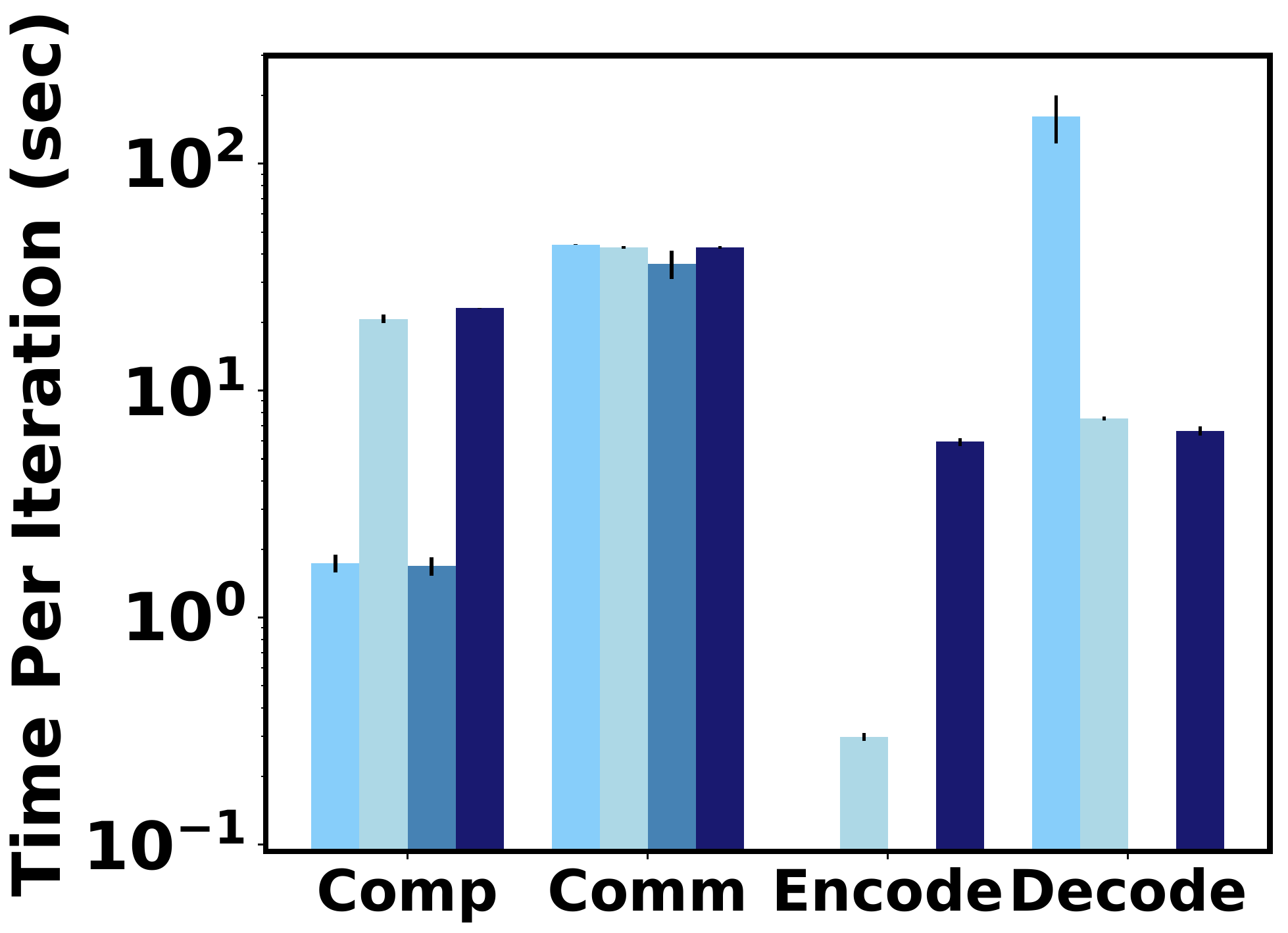}}
	\subfigure[VGG-19, Rev Grad]{\includegraphics[width=0.3\linewidth]{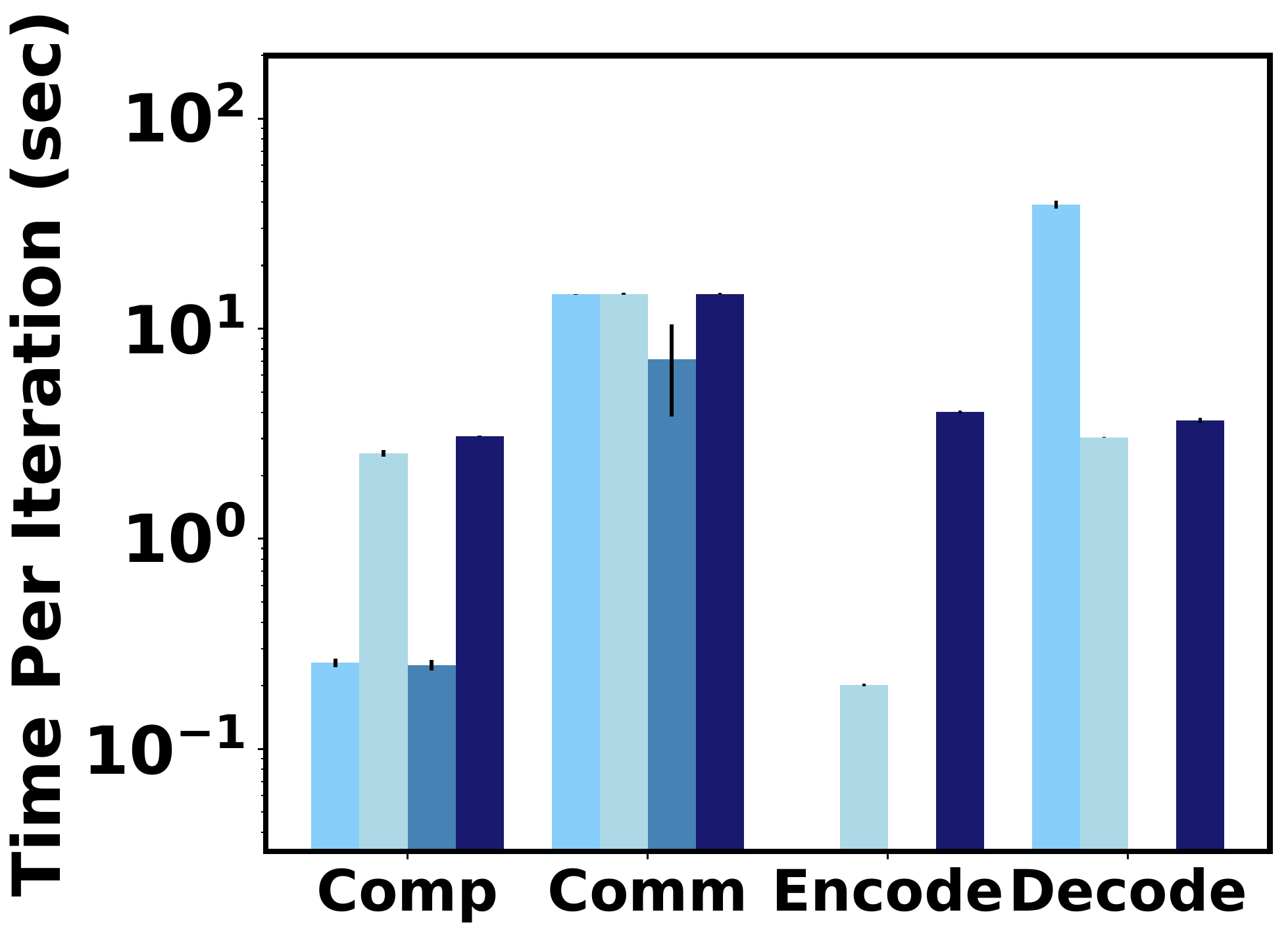}}
	\subfigure[AlexNet, Rev Grad]{\includegraphics[width=0.3\linewidth]{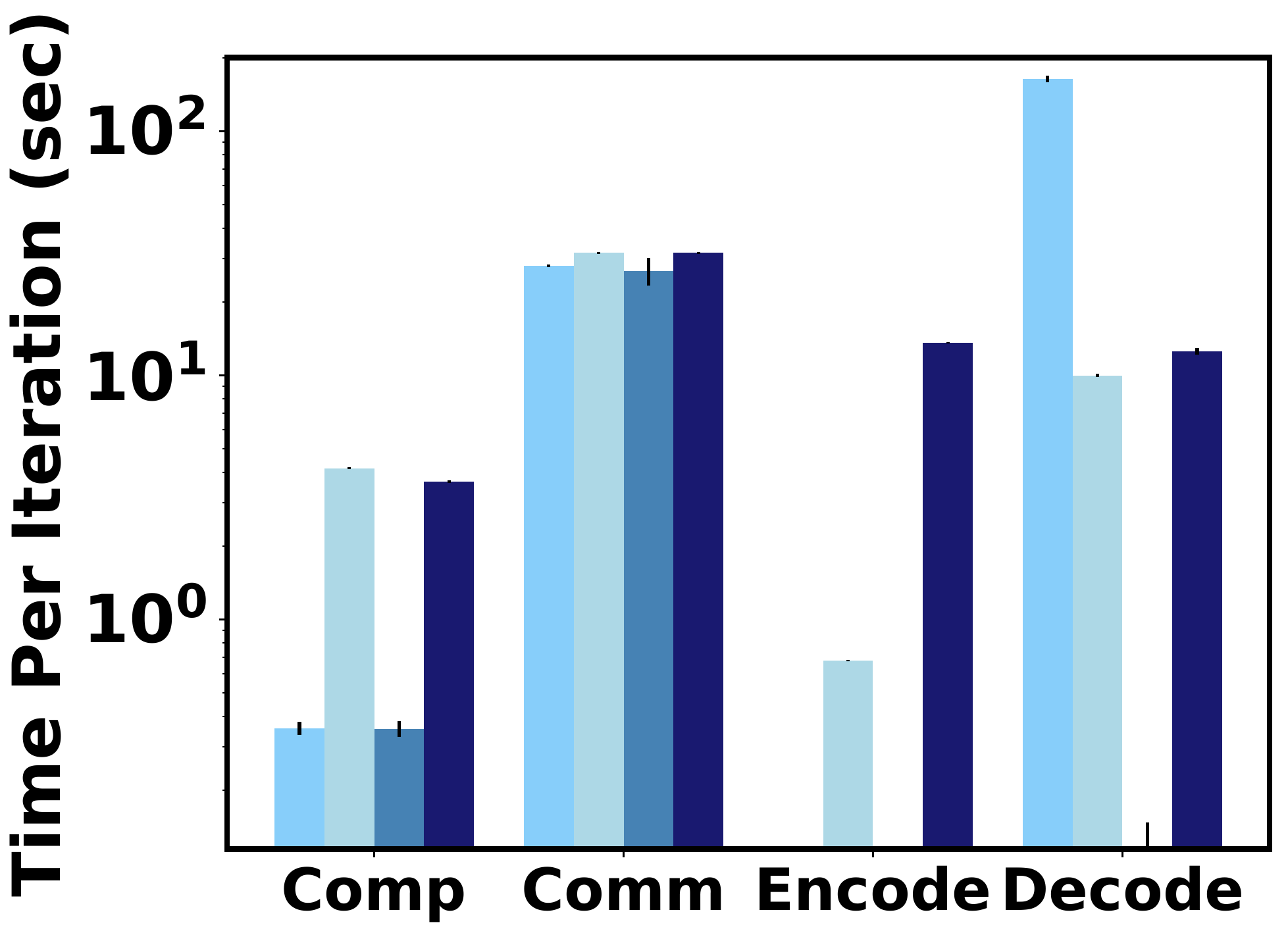}}
	\\
	\subfigure[ResNet-152, Const]{\includegraphics[width=0.3\linewidth]{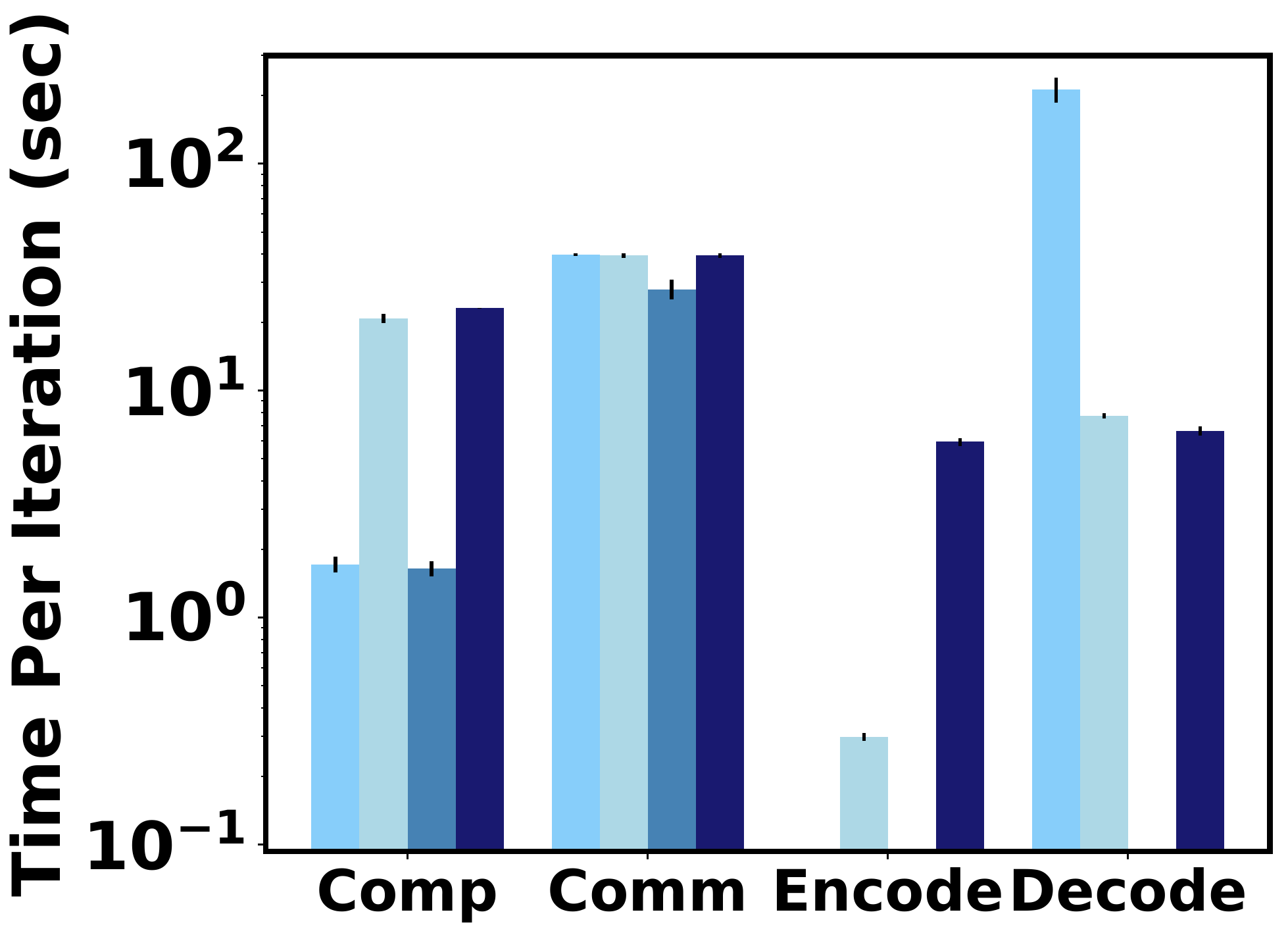}}
	\subfigure[VGG-19, Const]{\includegraphics[width=0.3\linewidth]{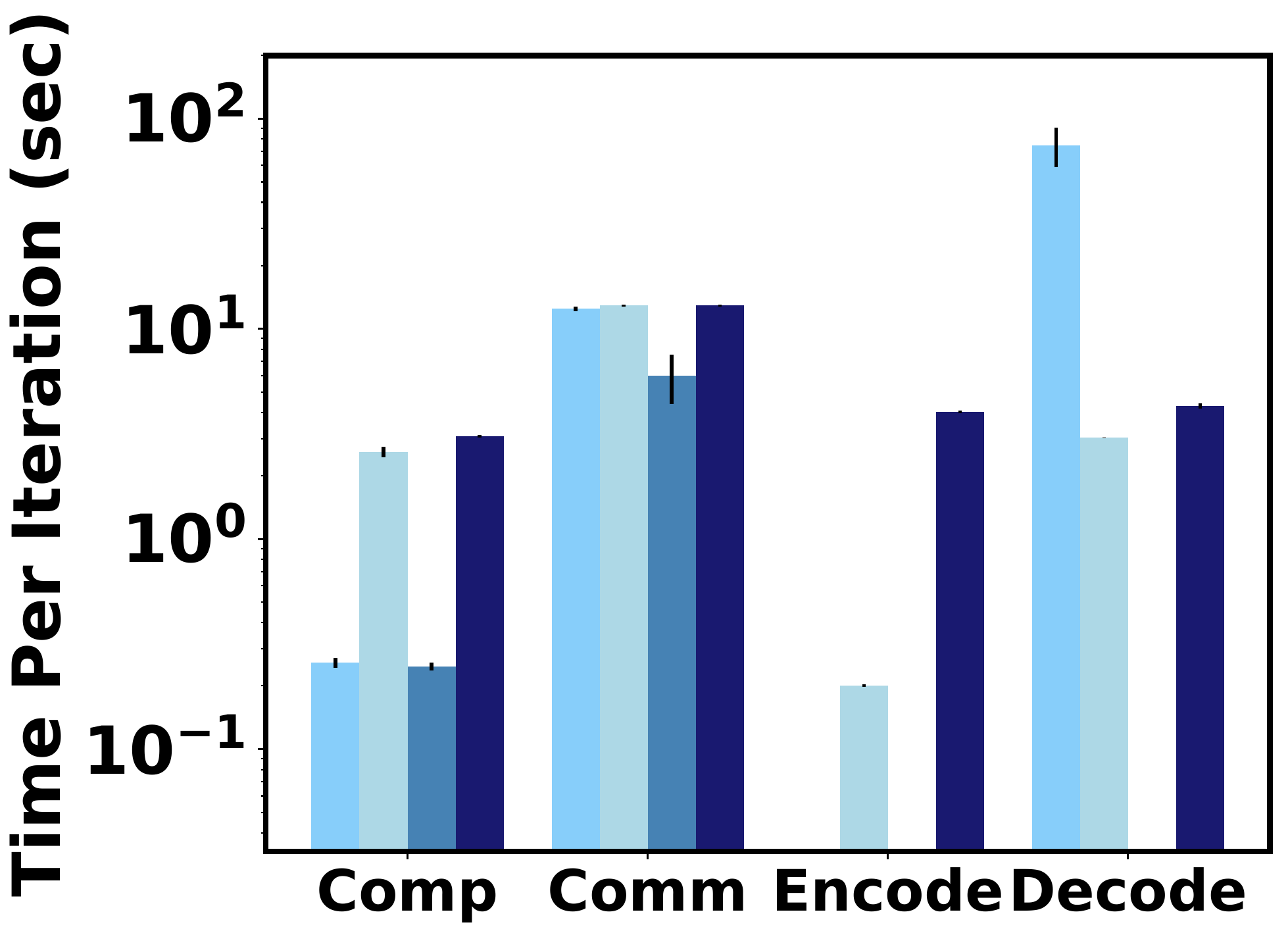}}
	\subfigure[AlexNet, Const]{\includegraphics[width=0.3\linewidth]{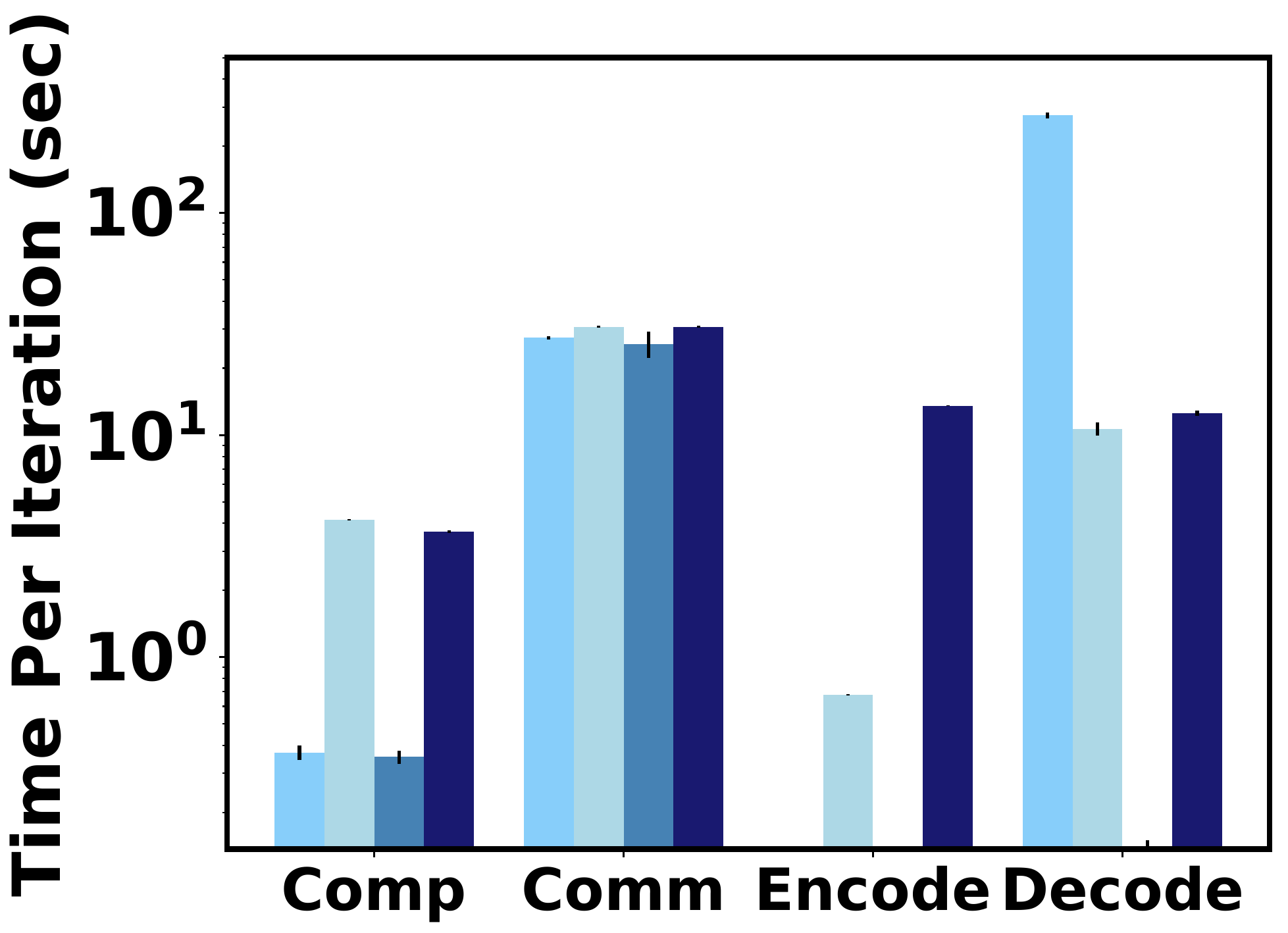}}
	\vspace{1mm}
	\caption{Empirical Per Iteration Time Cost on Large Models with 11.1\% adversarial nodes (a): reverse gradient adversary on ResNet-152, (b): reverse gradient adversary on VGG-19, (c): reverse gradient adversary on AlexNet, (d): constant adversary on ResNet-152, (e): constant adversary on VGG-19,  (f): constant adversary on AlexNet}
	\label{fig:3}
	\vspace{-2mm}
\end{figure*}

We provide empirical per iteration costs of applying \draco{} to three large state-of-the-art deep networks, ResNet-152, VGG-19, and AlexNet \cite{ResNet,VGG,AlexNet}.
The experiments provided here are run  on 46 real instances (45 compute nodes with 1 PS) on AWS EC2. For ResNet-152 and VGG-19, m4.4xlarge (equipped with 16 cores with 64 GB memory) instances are used while AlexNet experiments are run on m4.10xlarge (40 cores with 160 GB memory) instances given  the high memory cost during training. We use a batch size of $B=180$ and split the data among compute nodes. Therefore, each compute node is assigned $\frac{B}{n}=4$ data points per iteration. We use the Cifar10 dataset for all the aforementioned networks. For networks like AlexNet that were not designed for small images, we resize the Cifar10 images to fit the network.
As shown in Figure \ref{fig:3}, with $s=5$, the encoding and decoding time of \draco{} can be several times larger than the computation time of ordinary SGD, though SGD may not converge in adversarial settings. 
Nevertheless, \draco{} is still several times faster than GM. 
\begin{table}[htp]
	\centering
	\caption{Averaged Per Iteration Time Costs on ResNet-152 with 11.1\% adversary}
	\vspace{0.3 cm}

	\begin{tabular}{|c|c|c|c|c|}
		\hline Time Cost (sec)
		& Comp & Comm  & Encode  & Decode  \bigstrut\\
		\hline
		GM const & 1.72 & 39.74 & 0 & 212.31 \bigstrut\\
		\hline
		Rep const & 20.81 & 39.36 & 0.24 & 7.74 \bigstrut\\
		\hline
		SGD const & 1.64 & 27.99 & 0 & 0.09 \bigstrut\\
		\hline
		Cyclic const & 23.08 & 39.36 & 5.94 & 6.64 \bigstrut\\
		\hline
		GM rev grad & 1.73 & 43.98 & 0 & 161.29 \bigstrut\\
		\hline
		Rep rev grad & 20.71 & 42.86 & 0.29 & 7.54 \bigstrut\\
		\hline
		SGD rev grad & 1.69 & 36.27 & 0 & 0.09 \bigstrut\\
		\hline
		Cyclic rev grad & 23.08 & 42.86 & 5.95 & 6.65 \bigstrut\\
		\hline
		\end{tabular}%
		\label{tab:ResNetTuntime}%
		\end{table}%
		
		\begin{table}[htp]
			\centering
			\caption{Averaged Per Iteration Time Costs on VGG-19 with 11.1\% adversary}
			\vspace{0.3 cm}

			\begin{tabular}{|c|c|c|c|c|}
				\hline Time Cost (sec)
				& Comp & Comm  & Encode  & Decode  \bigstrut\\
				\hline
				GM const & 0.26 & 12.47 & 0 & 74.63 \bigstrut\\
				\hline
				Rep const & 2.59 & 12.91 & 0.20 & 3.03 \bigstrut\\
				\hline
				SGD const & 0.25 & 6.9 & 0 & 0.03 \bigstrut\\
				\hline
				Cyclic const & 3.08 & 12.91 & 4.01 & 4.30 \bigstrut\\
				\hline
				GM rev grad & 0.26 & 14.57 & 0 & 39.02 \bigstrut\\
				\hline
				Rep rev grad & 2.55 & 14.66 & 0.20 & 3.04 \bigstrut\\
				\hline
				SGD rev grad & 0.25 & 7.15 & 0 & 0.03 \bigstrut\\
				\hline
				Cyclic rev grad & 3.07 & 14.66 & 4.02 & 3.65 \bigstrut\\
				\hline
				\end{tabular}%

				\label{tab:VGGTuntime}%
				\end{table}%

Table \ref{tab:ResNetTuntime}, \ref{tab:VGGTuntime} and \ref{tab:AlexNetTuntime} provide the detailed cost of the runtime of each component of the algorithm in training ResNet-152, VGG-19 and AlexNet, respectively. 
While the communication cost is high in both \draco{} and the GM method, the decoding time of the GM approach, i.e., its geometric median update at the PS, is   prohibitively high. 
Meanwhile, the  encoding and decoding overhead of \draco{} is relatively negligible in these cases.

\begin{table}[htp]
	\centering
	\caption{Averaged Per Iteration Time Costs on AlexNet with 11.1\% adversarial nodes.}
	\vspace{0.3 cm}

	\begin{tabular}{|c|c|c|c|c|}
		\hline Time Cost (sec)
		& Comp & Comm  & Encode  & Decode  \bigstrut\\
		\hline
		GM const & 0.37 & 27.40 & 0 & 275.08 \bigstrut\\
		\hline
		Rep const & 4.16 & 30.71 & 0.67 & 10.65 \bigstrut\\
		\hline
		SGD const & 0.35 & 25.72 & 0 & 0.14 \bigstrut\\
		\hline
		Cyclic const & 3.67 & 30.71 & 13.55 & 12.54 \bigstrut\\
		\hline
		GM rev grad & 0.36 & 28.10 & 0 & 163.48 \bigstrut\\
		\hline
		Rep rev grad & 4.15 & 31.76 & 0.67 & 9.98 \bigstrut\\
		\hline
		SGD rev grad & 0.35 & 26.76 & 0 & 0.11 \bigstrut\\
		\hline
		Cyclic rev grad & 3.66 & 31.755 & 13.55 & 12.54 \bigstrut\\
		\hline
	\end{tabular}%
	
	\label{tab:AlexNetTuntime}%
\end{table}%
\paragraph{Effects of number of adversaries }

We also analyze how the number of adversaries affects the performance of \draco{}. We ran Cifar10 on ResNet-18 with 15 compute nodes, varying the number of adversaries $s$ from 1 to 7. 
For these experiments, we used the constant adversary model.
 For the repetition code, we adapted the group size based on $s$ while in the cyclic code we always took $2s+1$. Figure \ref{fig:4} shows the total runtime cost of \draco{} does not increase significantly as the number of adversaries increase. This is likely due to the fact that even at $s = 7$, the communication cost (which is not affected by the number of stragglers) is the dominant cost of the algorithm. 

\begin{figure}[htbp]
	\centering
	\subfigure[Repetition Code]{\includegraphics[width=0.4\linewidth]{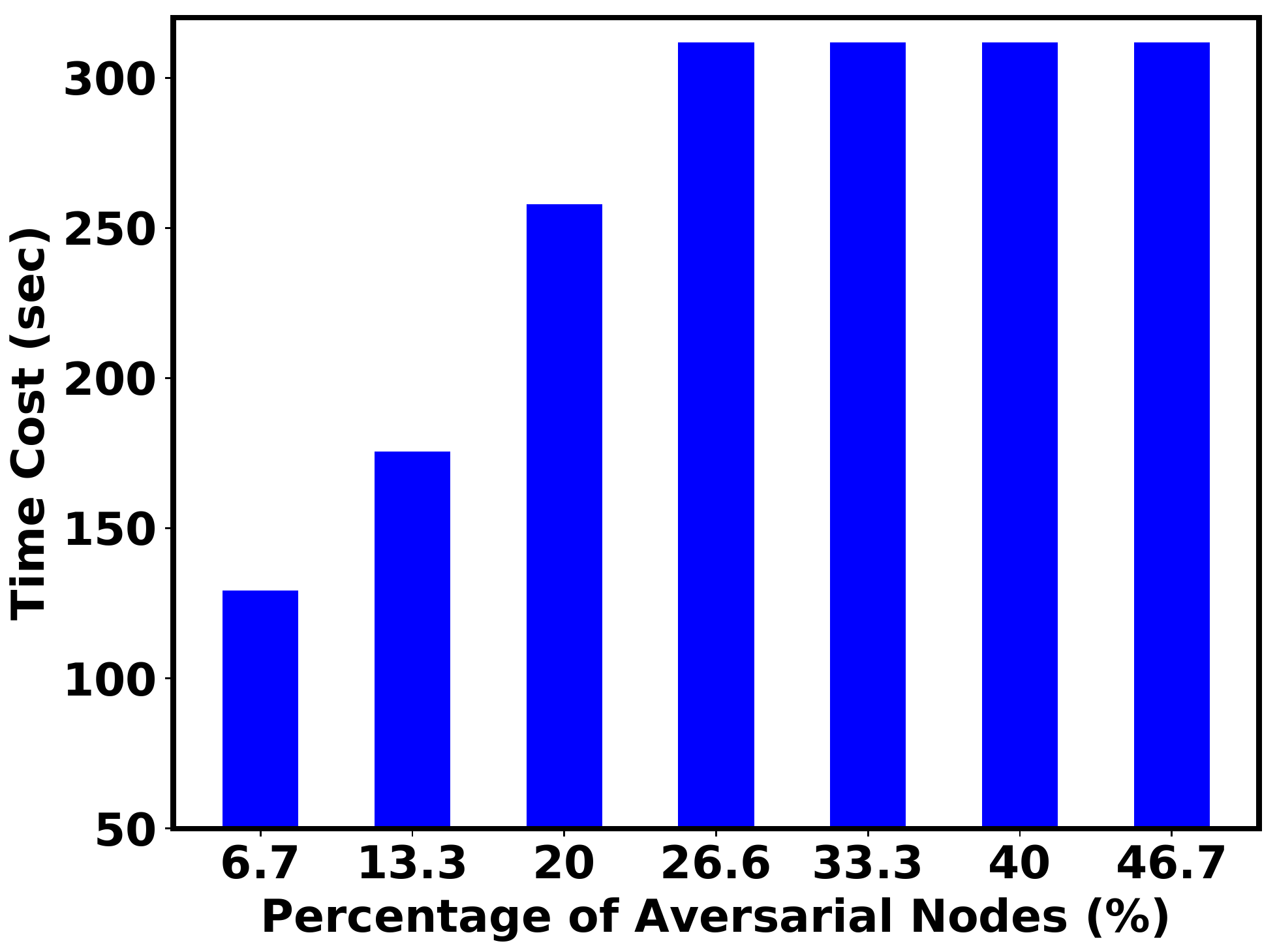}}
	\subfigure[Cyclic Code]{\includegraphics[width=0.4\linewidth]{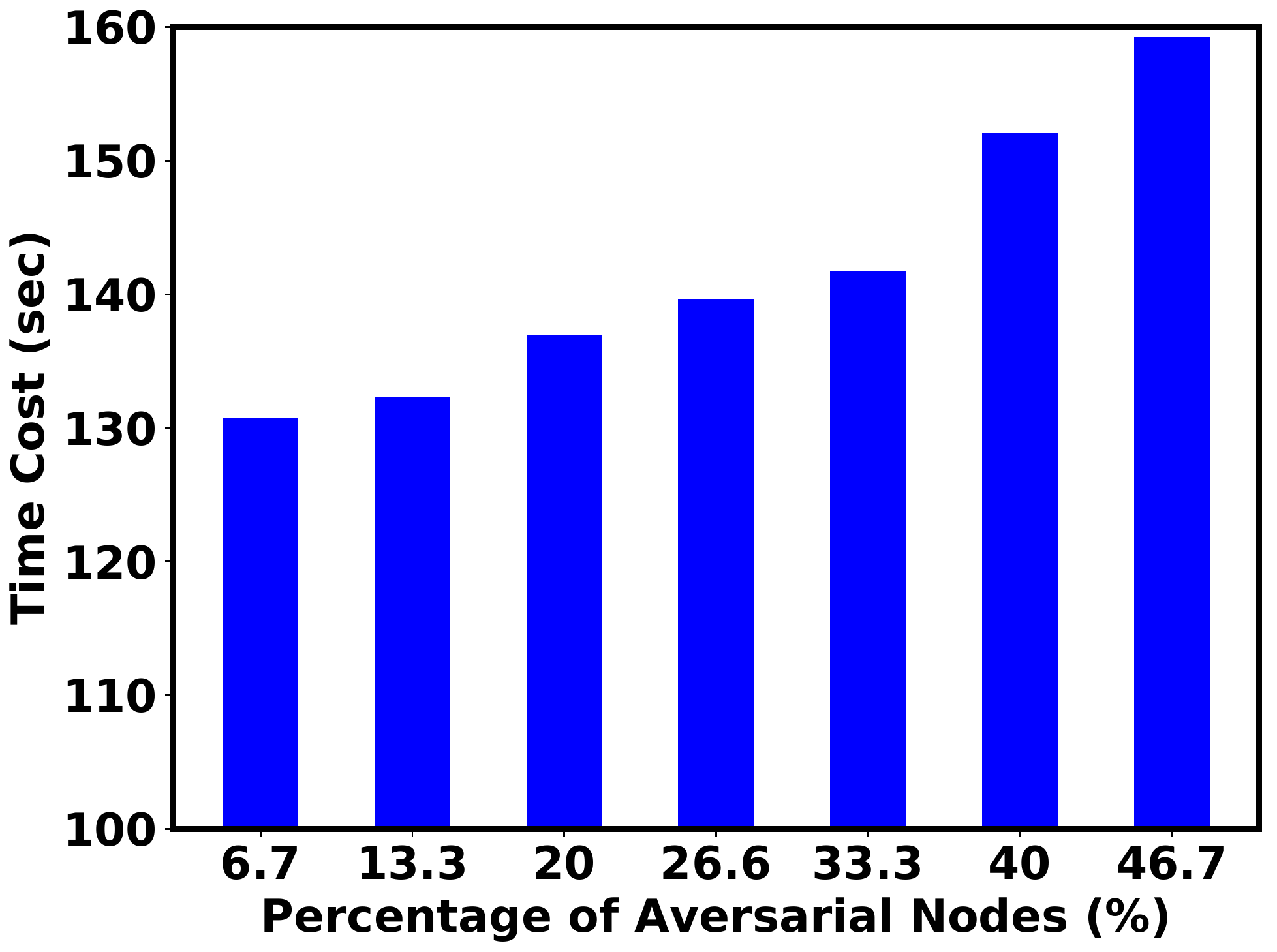}}
	\vspace{1mm}
	\caption{Time Cost to Reach 70\% Test set Accuracy with Cifar10 dataset run with ResNet-18 on cluster 15 computation nodes varying Percentage of Adversarial Nodes from 6.7\% to 46.7\% with Constant Adversary (a) Repetition Code and (b) Cyclic Code}
	\label{fig:4}
	\vspace{-2mm}
\end{figure}
\section{Conclusion and Open Problems}\label{Sec:Conclusion}
In this work we presented \draco{}, a framework for robust distributed training via algorithmic redundancy. 
\draco{} is robust to arbitrarily malicious compute nodes, while being orders of magnitude faster than state-of-the-art robust distributed systems. We give information--theoretic lower bounds on how much redundancy is required to resist adversaries while maintaining the correct update rule, and show that \draco{} achieves this lower bound. There are several interesting future directions.

First, \draco{} is designed to output the same model with or without adversaries. However, slightly inexact model updates often do not decrease performance noticeably. Therefore, we might ask whether we can either (1) tolerate more stragglers or (2) reduce the computational cost of \draco{} by only {\it approximately} recovering the desired gradient summation. Second, while we give two relatively efficient methods for encoding and decoding, there may be others that are more efficient for use in distributed setups. 
\section*{Acknowledgement}\label{Sec:Aknowledge}
This work was supported in part by a gift from Google and  AWS Cloud Credits for Research from Amazon.
We thank Jeffrey Naughton and Remzi Arpaci-Dusseau for  invaluable discussions and feedback on earlier drafts of this paper.
\bibliography{CodedML}
\bibliographystyle{alpha}

\newpage
\appendix
﻿﻿
\appendix

\section{Proofs}
\subsection{Proof of Theorem \ref{Thm:RedundancyRatioBound} }
For simplicity of proof, let us define a valid $s$-attack first.
\begin{definition}
	$\mathbf{N} =[\mathbf{n}_1,\mathbf{n}_2,\cdots, \mathbf{n}_P]$ is a valid $s$ attack if and only if $\left\vert{\{j: \|\mathbf{n}_{j}\|_0\not= 0 \}}\right\vert \leq s$. 
\end{definition}
Now we prove theorem \ref{Thm:RedundancyRatioBound}.
Suppose $(\mathbf{A}, E, D)$ can resist $s$ adversaries.
The goal is to prove 
$\|A\|_0 \geq P(2s+1)$.
    In fact we can prove a slightly stronger version: $\|\mathbf{A}_{\cdot,i}\|_0 \geq  \left(2s+1\right), i = 1,2,\cdots, B$.
	Suppose for some $i$,  $\|\mathbf{A}_{\cdot,i}\|_0 = \tau <  \left(2s+1\right)$.
	Without loss of generality, assume that $\mathbf{A}_{1, i}, \mathbf{A}_{2, i}, \mathbf{A}_{\tau,i} $ are non-zero. Let $\mathbf{G}_{-i} = [\mathbf{g}_1,\mathbf{g}_2,\cdots,\mathbf{g}_{i-1},\mathbf{g}_{i+1},\cdots, \mathbf{g}_P]$. 
	Since $(\mathbf{A},E,D)$ can protect against $s$ adversaries,  we have for any $\mathbf{G}$, 	\begin{align*}
		D(\mathbf{Z}^{\mathbf{A},E,\mathbf{G}}+\mathbf{N} ) = \mathbf{G} \mathbf{1}=  \mathbf{G}_{-i} \mathbf{1} + \mathbf{g}_i,
	\end{align*}
	for any valid $s$-attack $\mathbf{N}$.
	In particular, let $\mathbf{g}^1_i = \mathbf{1}_{d}$,  $\mathbf{g}^i_2 = -\mathbf{1}_{d}$,
	$\mathbf{G}^1 = [\mathbf{g}_1,\mathbf{g}_2,\cdots,\mathbf{g}_{i-1},\mathbf{g}_{i}^1,\mathbf{g}_{i+1},\cdots, \mathbf{g}_P]$, and $\mathbf{G}^2 = [\mathbf{g}_1,\mathbf{g}_2,\cdots,\mathbf{g}_{i-1},\mathbf{g}_{i}^2,\mathbf{g}_{i+1},\cdots, \mathbf{g}_P]$.  
	Then for any valid $s$ attack $\mathbf{N}^1, \mathbf{N}^2$,
	\begin{align*}
		D(\mathbf{Z}^{\mathbf{A},E,\mathbf{G}^1}+\mathbf{N}^1 ) =  \mathbf{G}_{-i} \mathbf{1}_{P-1} + \mathbf{1}_d.
	\end{align*}	
	and
	\begin{align*}
		D(\mathbf{Z}^{\mathbf{A},E,\mathbf{G}^2}+\mathbf{N}^2 ) =  \mathbf{G}_{-i} \mathbf{1}_{P-1} - \mathbf{1}_d.
	\end{align*}	
	Our goal is to find $\mathbf{N}^1, \mathbf{N}^2$ such that $D(\mathbf{Z}^{\mathbf{A},E,\mathbf{G}^1}+\mathbf{N}^1 ) = D(\mathbf{Z}^{\mathbf{A},E,\mathbf{G}^2}+\mathbf{N}^2 )$ which then will lead to a contradiction.
	Construct $\mathbf{N}^1$ and $\mathbf{N}^2$ by
	\begin{equation*}
	\mathbf{N}^1_{\ell,j}= 
	\begin{cases}
	\left[\mathbf{Z}^{\mathbf{A},E, \mathbf{N}^2}\right]_{\ell,j} - 	\left[\mathbf{Z}^{\mathbf{A},E, \mathbf{N}^1}\right]_{\ell,j} ,&j=1,2,\cdots, \ceil{\frac{\tau-1}{2}}\\
	0,              & \text{otherwise}
	\end{cases}
	\end{equation*}
	and
	\begin{equation*}
	\mathbf{N}^2_{\ell,j}= 
	\begin{cases}
		\left[\mathbf{Z}^{\mathbf{A},E, \mathbf{N}^1}\right]_{\ell,j} - 	\left[\mathbf{Z}^{\mathbf{A},E, \mathbf{N}^2}\right]_{\ell,j}, & j= \ceil{\frac{\tau-1}{2}}, \ceil{\frac{\tau-1}{2}}+1,\cdots, \tau\\
	0,              & \text{otherwise}
	\end{cases}
	\end{equation*}
	One can easily verify that $\mathbf{N}^1, \mathbf{N}^2$ are both valid $s$ attack. Meanwhile, we have
	\begin{align*}
		\left[\mathbf{Z}^{\mathbf{A},E,\mathbf{G}^1}\right]_{\ell,j}+\mathbf{N}^1_{{\ell,j}} = \left[\mathbf{Z}^{\mathbf{A},E,\mathbf{G}^2}\right]_{\ell,j}+\mathbf{N}^2_{{\ell,j}}, j=1,2,\cdots, \tau
	\end{align*}	 
	due to the above construction of $\mathbf{N}^1, \mathbf{N}^2$.
	Note that $\mathbf{A}_{j,i}=0$ for all $j>\tau$, which implies that for all compute nodes with index $j>\tau$, their encoder functions do not depend on the $i$th gradient. Since $\mathbf{G}^1$ and $\mathbf{G}^2$ only differ in the $i$th gradient, the encoder function of any compute node with index $j>\tau$ should have the same output.
	Thus, we have  
	\begin{align*}
		\left[\mathbf{Z}^{\mathbf{A},E,\mathbf{G}^1}\right]_{\ell,j}+\mathbf{N}^1_{{\ell,j}} = \left[\mathbf{Z}^{\mathbf{A},E,\mathbf{G}^1}\right]_{\ell,j} = \left[\mathbf{Z}^{\mathbf{A},E,\mathbf{G}^2}\right]_{\ell,j} = \left[\mathbf{Z}^{\mathbf{A},E,\mathbf{G}^2}\right]_{\ell,j}+\mathbf{N}^2_{{\ell,j}}, j> \tau
	\end{align*}
	Hence, we have 		
	\begin{align*}
		\left[\mathbf{Z}^{\mathbf{A},E,\mathbf{G}^1}\right]_{\ell,j}+\mathbf{N}^1_{{\ell,j}} =  \left[\mathbf{Z}^{\mathbf{A},E,\mathbf{G}^2}\right]_{\ell,j}+\mathbf{N}^2_{{\ell,j}}, \forall j
	\end{align*}
	which means
	\begin{align*}
		\mathbf{Z}^{\mathbf{A},E,\mathbf{G}^1}+\mathbf{N}^1 = \mathbf{Z}^{\mathbf{A},E,\mathbf{G}^2}+\mathbf{N}^2
	\end{align*}
Therefore, we have 
\begin{align*}
    D(\mathbf{Z}^{\mathbf{A},E,\mathbf{G}^1}+\mathbf{N}^1 ) = D(\mathbf{Z}^{\mathbf{A},E,\mathbf{G}^2}+\mathbf{N}^2 )
\end{align*}
	and thus
\begin{align*}
    \mathbf{G}_{-1} \mathbf{1}_{P-1} + \mathbf{1}_d = D(\mathbf{Z}^{\mathbf{A},E,\mathbf{G}^1}+\mathbf{N}^1 ) = D(\mathbf{Z}^{\mathbf{A},E,\mathbf{G}^2}+\mathbf{N}^2 ) = \mathbf{G}_{-1} \mathbf{1}_{P-1} - \mathbf{1}_d 
\end{align*}
This gives us a contradiction. Hence, the assumption is not correct and we must have $\|\mathbf{A}_{\cdot,i}\|_0 \geq  \left(2s+1\right), i = 1,2,\cdots, P$.	
Thus, we must have $\|A\|_0 \geq (2s+1)P$. \qed
	
A direct but interesting corollary of this theorem is a bound on the number of adversaries \draco{} can resist.

\begin{corollary}\label{Cor:AdversarialBound}
	$(\mathbf{A}, E, D)$ can resist at most $\frac{P-1}{2}$ adversarial nodes.
\end{corollary}
\begin{proof}
According to Theorem \ref{Thm:RedundancyRatioBound}, the redundancy ratio is at least $2s+1$, meaning that every data point must be replicated by at least $2s+1$. 
Since there are  $P$ compute node in total, we must have $2s+1\leq P$, which implies $s\leq \frac{P-1}{2}$. 
Thus, $(\mathbf{A},E,D)$ can resist at most $\frac{P-1}{2}$ adversaries.
\end{proof}
In other words, at least a majority of the compute nodes must be non-adversarial. \qed

\subsection{Proof of Theorem \ref{Thm:RepCodeOptimality} }
Since there are at most $s$ adversaries, there are at least $2s+1-s=s+1$ non-adversarial compute nodes in each group.
Thus, performing majority vote on each group returns the correct gradient, and thus the repetition code guarantees that the result is correct.
The complexity at each compute node is clearly $\mathcal{O}((2s+1)d)$ since each of them only computes the sum of $(2s+1)$ $d$-dimensional gradients.
For the decoder at the PS, within each group of $(2s+1)$ machine, it takes $\mathcal{O}((2s+1) d )$ computations to find the majority. Since there are $\frac{P}{(2s+1)}$ groups, it takes in total $\mathcal{O}((2s+1) d \frac{P}{(2s+1)} ) = \mathcal{O}(P d )$ computations. 
Thus, this achieves linear-time encoding and decoding. \qed

\subsection{Proof of Lemma \ref{Thm:CycCodeColumnSpan} }
We first prove that $\mathbf{A}_{j,k} = 0 \Rightarrow \mathbf{W}_{j,k} = 0$.

Suppose $\mathbf{A}_{j,k} = 0$ for some $j,k$.
Then by definition $k \in \alpha_j$.
By $\mathbf{0} = \begin{matrix} \begin{bmatrix}
\mathbf{q}_j & 1
\end{bmatrix}\end{matrix} \cdot  \left[\mathbf{C}_{L}\right]_{\cdot,\alpha_j} $ we have $ 0 =\begin{matrix} \begin{bmatrix}
\mathbf{q}_j & 1
\end{bmatrix}\end{matrix} \left[\mathbf{C}_L\right]_{\cdot,k} =\mathbf{W}_{j,k}$. 

Next we prove that for any index set $U$ such that $|U|\geq P-(2s+1)$, the column span of $\mathbf{W}_{\cdot,U}$ contains $\mathbf{1}$.
This is equivalent to that for any index set $U$ such that $|U|\geq P-(2s+1)$, there exists a vector $\mathbf{b}$ such that $\mathbf{W}_{\cdot,U} b = \mathbf{1}$. 
Now we show such $b$ exists.
Note that $\mathbf{C}_{L}$ is a $(P-2s)\times P$ full rank Vandermonde matrix and thus any $P-2s$ columns of 
$\mathbf{C}_{L}$ are linearly independent.
Let $\bar{U}$ be the first $P-2s$ elements in $U$.
Then all columns of $\left[\mathbf{C}_{L}\right]_{\cdot,\bar{U}}$ are linearly independent and thus $\left[\mathbf{C}_{L}\right]_{\cdot, \bar{U}}$ is invertible.
Let $\mathbf{b}_{\bar{U}} \triangleq \mathbf{\bar{b}} = \left(C^L_{\bar{U}}\right)^{-1} \begin{matrix} \begin{bmatrix}
0 & 0 & \cdots & 0 & 1
\end{bmatrix}\end{matrix}^T $.
For any $j \not\in \bar{U}$, let
 $\mathbf{b}_{j} = 0$.
Then we have 

\begin{equation*}
\begin{split}
\mathbf{W}_{U}  \mathbf{b} & = \begin{matrix} \begin{bmatrix}
		\mathbf{Q} & \mathbf{1}
	\end{bmatrix}\end{matrix} \times  \left[\mathbf{C}_L\right]_{\cdot,U} \mathbf{b} \\
	& =	\begin{matrix} \begin{bmatrix}
			\mathbf{Q} & \mathbf{1}
		\end{bmatrix}\end{matrix} \times  \left[\mathbf{C}_L\right]_{\cdot,\bar{U}} \mathbf{\bar{b}} \\	
	& = \begin{matrix} \begin{bmatrix}
			\mathbf{Q} & \mathbf{1}
		\end{bmatrix}\end{matrix} \times  \left[\mathbf{C}_L\right]_{\cdot,\bar{U}} \times \left[\mathbf{C}_L\right]_{\cdot,\bar{U}}^{-1} \begin{matrix} \begin{bmatrix}
			0 & 0 & \cdots & 0 & 1
		\end{bmatrix}\end{matrix}^T\\	
	& = \begin{matrix} \begin{bmatrix}
			\mathbf{Q} & \mathbf{1}
		\end{bmatrix}\end{matrix} \begin{matrix} \begin{bmatrix}
		0 & 0 & \cdots & 0 & 1
	\end{bmatrix}\end{matrix}^T\\
	& = \mathbf{1}.\\				
\end{split}
\end{equation*}
This completes the proof. \qed

\subsection{Proof of Lemma \ref{Thm:CycCodeDetection} }
We need a few lemmas first.
\begin{lemma}\label{Lemma:CycProb}
	Let a $P$-dimensional vector $\mathbf{\gamma} \triangleq \left[\gamma_1, \gamma_2,\cdots, \gamma_P \right]^T = \left(\mathbf{f} \mathbf{N}\right)^T$. Then we have 
	\begin{align*}
	\textit{Pr}( \{ j : \gamma_j \not= 0 \} = \{j: \| \mathbf{N}_{\cdot,j} \|_0 \not= 0 \}) = 1.
	\end{align*}
\end{lemma}
\begin{proof}
Let us prove that 
	\begin{align*}
		\textit{Pr}( \mathbf{N}_{\cdot,j} \not= 0 \} | \gamma_j \not = 0 ) = 1.
	\end{align*}
	and 
	\begin{align*}
		\textit{Pr}( \gamma_j \not = 0 | \mathbf{N}_{\cdot,j} \not= 0 \} ) = 1.
	\end{align*}
	for any $j$.	
	Combining those two equations we prove the lemma.
	
	The first equation is straightforward. Suppose $\mathbf{N}_{\cdot,j} = 0$.
	Then we immediately have $\gamma_j = \mathbf{f} \mathbf{N}_{\cdot,j} = 0$.
	For the second one, note that $\mathbf{f}$ has entries drawn independently from the standard normal distribution. Therefore we have that $\gamma_j = \mathbf{f} \mathbf{N}_{\cdot,j} \sim \mathcal{N}(\mathbf{1}^T \mathbf{N}_{\cdot,j}, \|\mathbf{N}_{\cdot,j}\|_2^2 )$.
	Since $\gamma_j$ is a random variable with normal distribution, the probability of it being any particular value is $0$.
	In particular, 
	\begin{align*}
		\textit{Pr}( \gamma_j = 0 | \mathbf{N}_{\cdot,j} \not= 0 \} ) = 0,
	\end{align*}	
	and thus 
	\begin{align*}
		\textit{Pr}( \gamma_j \not= 0 | \mathbf{N}_{\cdot,j} \not= 0 \} ) = 1
	\end{align*}	
	which proves the second equation and finishes the proof.
\end{proof}

\begin{lemma}\label{Lemma:Orthogonal}
	$\mathbf{R}^{Cyc} \mathbf{C}_{R}^\dag = \mathbf{N} \mathbf{C}_{R}^\dag$.
\end{lemma}
\begin{proof}
By definition, $\mathbf{R}^{Cyc} \mathbf{C}_{R}^\dag = \left( \mathbf{G} \mathbf{W} + \mathbf{N} \right) \mathbf{C}_{R}^\dag = \left( \mathbf{G} \begin{matrix} \begin{bmatrix}
\mathbf{Q} & \mathbf{1}
\end{bmatrix}\end{matrix}  \mathbf{C}_L + \mathbf{N} \right) \mathbf{C}_{R}^\dag = \mathbf{G} \begin{matrix} \begin{bmatrix}
\mathbf{Q} & \mathbf{1}
\end{bmatrix}\end{matrix}   \mathbf{C}_L \mathbf{C}_{R}^\dag + \mathbf{N} \mathbf{C}_{R}^\dag  = \mathbf{N} \mathbf{C}_{R}^\dag$.
In the last equation we use the fact that IDFT matrix is unitary and thus $\mathbf{C}_L \mathbf{C}_{R}^\dag = \mathbf{0}_{(P-2s)\times(2s)}$.
\end{proof}

\begin{lemma}\label{Lemma:DFT}
	Let a $P$-dimensional vector $\hat{\mathbf{h}} \triangleq [\hat{h}_0, \hat{h}_1,\cdots, \hat{h}_{P-1}]^T$ be the discrete Fourier transformation (DFT) of a $P$-dimensional vector $\hat{\mathbf{t}} \triangleq [\hat{t}_{1},\hat{t}_{2},\cdots,\hat{t}_{P-1}]^T$ which has at most $s$ non-zero elements, i.e., $\hat{\mathbf{h}} = \mathbf{C}^{\dag} \hat{\mathbf{t}}$ and $\| \mathbf{t} \|_0 \leq s$.
	Then there exists a $s$-dimensional vector $\hat{\beta} \triangleq  [\hat{\beta}_{0},\hat{\beta}_{1},\cdots, \hat{\beta}_{s-1}]^T$, such that
	\begin{align}\label{eq:RecoverLinearSystem}
	\begin{matrix}
		\begin{bmatrix}
			\hat{h}_{P-s-1}      & \hat{h}_{P-s} & \dots & \hat{h}_{P-2} \\
			\hat{h}_{P-s-2}       & \hat{h}_{P-s-1} & \dots & \hat{h}_{P-3} \\
			\hdots & \hdots & \ddots &\vdots \\
			\hat{h}_{P-2s}  & \hat{h}_{P-s+1} & \dots & \hat{h}_{P-s-1}
		\end{bmatrix}
		\hat{\beta}
		=
		\begin{bmatrix}
			\hat{h}_{P-1}\\
			\hat{h}_{P-2}\\
			\vdots\\
			\hat{h}_{P-s}
		\end{bmatrix}
	\end{matrix}.
\end{align}
Furthermore, for any $\hat{\beta}$ satisfying the above equations, 
\begin{align}\label{eq:InductionEquation}
	\hat{h}_{\ell} = \sum_{u=0}^{s-1} \hat{\beta}_{u} \hat{h}_{\ell+u-s}, 
\end{align}
always holds for all $\ell$, where $\hat{h}_{\ell} = \hat{h}_{P+\ell}$.
\end{lemma}
\begin{proof}
Let $i_1,i_2,\cdots, i_s$ be the index of the non-zero elements in $\hat{\mathbf{t}}$.
Let us define the location polynomial $p(\omega) = \prod_{k=1}^{s}(\omega-e^{-\frac{2\pi i}{P}i_k}  )\triangleq \sum_{k=0}^{s} \theta_k \omega^k$, where $\theta_s = 1$.
Let a $s$-dimensional vector $\hat{\beta}^* \triangleq -[\theta_0,\theta_1,\cdots, \theta_{s-1}]^T$.

Now we prove that  $\hat{\beta}=\hat{\beta}^*$ is a solution to the system of linear equations \eqref{eq:RecoverLinearSystem}.
To see this, note that by definition, for any $\lambda$, we have $0 = p(e^{-\frac{2\pi i}{P}i_\lambda}) = \sum_{k=0}^{s} \theta_k e^{-\frac{2\pi i}{P}i_\lambda k}$.
Multiply both  side by $\hat{t}_{i_\lambda} e^{-\frac{2\pi i}{P}i_\lambda \eta }$, we have
\begin{align*}
0 &= \hat{t}_{i_\lambda} e^{-\frac{2\pi i}{P}i_\lambda \eta} \sum_{k=0}^{s} \theta_k e^{-\frac{2\pi i}{P}i_\lambda k}\\
&=\hat{t}_{i_\lambda} \sum_{k=0}^{s} \theta_k e^{-\frac{2\pi i}{P}i_\lambda (k+\eta)}.
\end{align*}
Summing over $\lambda$, we have 
\begin{align*}
	0 &=\sum_{\lambda = 1}^{s}\hat{t}_{i_\lambda} \sum_{k=0}^{s} \theta_k e^{-\frac{2\pi i}{P}i_\lambda (k+\eta)}\\
	&=\sum_{k=0}^{s} \theta_k \sum_{\lambda = 1}^{s}\hat{t}_{i_\lambda}   e^{-\frac{2\pi i}{P}i_\lambda (k+\eta)}.
\end{align*}
By definition, $\hat{h}_j = \mathbf{C}_{j,\cdot} \hat{\mathbf{t}} = \frac{1}{\sqrt{P}}\sum_{k=0}^{P-1} e^{-\frac{2\pi i }{P} j k} \hat{t}_k  =  \frac{1}{\sqrt{P}}\sum_{\lambda=1}^{s} \hat{t}_{i_\lambda} e^{-\frac{2\pi i}{P}  i_\lambda j} $.
Hence, the above equation becomes
\begin{align*}
	0 &=\sum_{k=0}^{s} \theta_k \sqrt{P} \hat{h}_{k+\eta}
\end{align*}
which is equivalent to
\begin{align*}
	\hat{h}_{s+\eta} &=\sum_{k=0}^{s-1} -\theta_k \hat{h}_{k+\eta}
\end{align*}
due to the fact that $\theta_s = 1$. 
By setting $\eta=-s+P-1,-s+P-2,\cdots,-s+P-s$, one can easily see that the above equation becomes identical to the system of linear equations in \eqref{eq:RecoverLinearSystem} with $\hat{\beta} =\hat{\beta}^*=-[\theta_0,\theta_1,\cdots, \theta_{s-1}]^T$.

Now let us prove for any $\hat{\beta}$ that satisfies equation \eqref{eq:RecoverLinearSystem}, we have \eqref{eq:InductionEquation}. 
Note that an equivalent form of \eqref{eq:InductionEquation} is that the following system of linear equations

	\begin{align}\label{eq:CycInduction}
		\begin{matrix}
			\begin{bmatrix}
				\hat{h}_{P-s-1+\ell}      & \hat{h}_{P-s+\ell} & \dots & \hat{h}_{P-2+\ell} \\
				\hat{h}_{P-s-2+\ell}       & \hat{h}_{P-s-1+\ell} & \dots & \hat{h}_{P-3+\ell} \\
				\hdots & \hdots & \ddots &\vdots \\
				\hat{h}_{P-2s+\ell}  & \hat{h}_{P-s+1+\ell} & \dots & \hat{h}_{P-s-1+\ell}
			\end{bmatrix}
			\hat{\beta}
			=
			\begin{bmatrix}
				\hat{h}_{P-1+\ell}\\
				\hat{h}_{P-2+\ell}\\
				\vdots\\
				\hat{h}_{P-s+\ell}
			\end{bmatrix}
		\end{matrix}
	\end{align}
	holds for $\ell=0,1,2\cdots,P-1$.
	We prove this by induction.
	When $\ell=1$, this is true since $\hat{\beta}$ satisfies the system of linear equations in \eqref{eq:RecoverLinearSystem}.
	Assume it holds for $\ell = \mu$, i.e.,
	\begin{align*}
		\begin{matrix}
			\begin{bmatrix}
				\hat{h}_{P-s-1+\mu}      & \hat{h}_{P-s+\mu} & \dots & \hat{h}_{P-2+\mu} \\
				\hat{h}_{P-s-2+\mu}       & \hat{h}_{P-s-1+\mu} & \dots & \hat{h}_{P-3+\mu} \\
				\hdots & \hdots & \ddots &\vdots \\
				\hat{h}_{P-2s+\mu}  & \hat{h}_{P-s+1+\mu} & \dots & \hat{h}_{P-s-1+\mu}
			\end{bmatrix}
			\hat{\beta}
			=
			\begin{bmatrix}
				\hat{h}_{P-1+\mu}\\
				\hat{h}_{P-2+\mu}\\
				\vdots\\
				\hat{h}_{P-s+\mu}
			\end{bmatrix}
		\end{matrix}
	\end{align*}	
	Now we need to prove it also holds  when $\ell=\mu+1$, \ie 
	\begin{align*}
		\begin{matrix}
			\begin{bmatrix}
				\hat{h}_{P-s-1+\mu+1}      & \hat{h}_{P-s+\mu+1} & \dots & \hat{h}_{P-2+\mu+1} \\
				\hat{h}_{P-s-2+\mu+1}       & \hat{h}_{P-s-1+\mu+1} & \dots & \hat{h}_{P-3+\mu+1} \\
				\hdots & \hdots & \ddots &\vdots \\
				\hat{h}_{P-2s+\mu+1}  & \hat{h}_{P-s+1+\mu+1} & \dots & \hat{h}_{P-s-1+\mu+1}
			\end{bmatrix}
			\hat{\beta}
			=
			\begin{bmatrix}
				\hat{h}_{P-1+\mu+1}\\
				\hat{h}_{P-2+\mu+1}\\
				\vdots\\
				\hat{h}_{P-s+\mu+1}
			\end{bmatrix}
		\end{matrix}.
	\end{align*}
	First, since both $\hat{\beta}, \hat{\beta}^*$ satisfy the induction assumption,	we must have 
	\begin{align*}
		\begin{matrix}
			\begin{bmatrix}
				\hat{h}_{P-s-1+\mu}      & \hat{h}_{P-s+\mu} & \dots & \hat{h}_{P-2+\mu} \\
				\hat{h}_{P-s-2+\mu}       & \hat{h}_{P-s-1+\mu} & \dots & \hat{h}_{P-3+\mu} \\
				\hdots & \hdots & \ddots &\vdots \\
				\hat{h}_{P-2s+\mu}  & \hat{h}_{P-s+1+\mu} & \dots & \hat{h}_{P-s-1+\mu}
			\end{bmatrix}
			\end{matrix}
			(\hat{\beta} -\hat{\beta}^*)
			=\mathbf{0}_{s}.
	\end{align*}	
	Due to the induction assumption, one can verify that 
	\begin{align*}
		[\theta_{s-1},\theta_{s-2},\cdots, \theta_{0}]
		\begin{matrix}
			\begin{bmatrix}
				\hat{h}_{P-s-1+\mu}      & \hat{h}_{P-s+\mu} & \dots & \hat{h}_{P-2+\mu} \\
				\hat{h}_{P-s-2+\mu}       & \hat{h}_{P-s-1+\mu} & \dots & \hat{h}_{P-3+\mu} \\
				\hdots & \hdots & \ddots &\vdots \\
				\hat{h}_{P-2s+\mu}  & \hat{h}_{P-s+1+\mu} & \dots & \hat{h}_{P-s-1+\mu}
			\end{bmatrix}
		\end{matrix}
		= 
		\begin{matrix}
			\begin{bmatrix}
		\hat{h}_{P-s+\mu}  & \hat{h}_{P-s+\mu+1} & \cdots
		\hat{h}_{P-2+\mu+1} 
			\end{bmatrix}
		\end{matrix},
	\end{align*}	
	and thus we have 
	\begin{align*}
		&\begin{matrix}
			\begin{bmatrix}
				\hat{h}_{P-s+\mu}  & \hat{h}_{P-s+\mu+1} & \cdots
				\hat{h}_{P-2+\mu+1} 
			\end{bmatrix}
		\end{matrix}
		(\hat{\beta} -\hat{\beta}^*)
		\\ = 	&[\theta_{s-1},\theta_{s-2},\cdots, \theta_{0}]
		\begin{matrix}
			\begin{bmatrix}
				\hat{h}_{P-s-1+\mu}      & \hat{h}_{P-s+\mu} & \dots & \hat{h}_{P-2+\mu} \\
				\hat{h}_{P-s-2+\mu}       & \hat{h}_{P-s-1+\mu} & \dots & \hat{h}_{P-3+\mu} \\
				\hdots & \hdots & \ddots &\vdots \\
				\hat{h}_{P-2s+\mu}  & \hat{h}_{P-s+1+\mu} & \dots & \hat{h}_{P-s-1+\mu}
			\end{bmatrix}
		\end{matrix}
		(\hat{\beta} -\hat{\beta}^*)
		=0.
	\end{align*}		
Hence,
\begin{align*}
		& \begin{matrix}
			\begin{bmatrix}
				\hat{h}_{P-s+\mu}      & \hat{h}_{P-s-1+\mu} & \dots & \hat{h}_{P-1+\mu} \\
			\end{bmatrix}
		\end{matrix}
			\hat{\beta}\\
			= & 	\begin{matrix}
				\begin{bmatrix}
					\hat{h}_{P-s+\mu}      & \hat{h}_{P-s-1+\mu} & \dots & \hat{h}_{P-1+\mu} \\
				\end{bmatrix}
			\end{matrix}
			\hat{\beta}^* + \begin{matrix}
				\begin{bmatrix}
					\hat{h}_{P-s+\mu}      & \hat{h}_{P-s-1+\mu} & \dots & \hat{h}_{P-1+\mu} \\
				\end{bmatrix}
			\end{matrix}
			(\hat{\beta} - \hat{\beta}^*) 
			=\hat{h}_{P+\mu} = \hat{h}_{P-1+\mu+1}.
	\end{align*}	
	Furthermore, by induction assumption, we have
	\begin{align*}
		& \begin{matrix}
			\begin{bmatrix}
				\hat{h}_{P-s-2+\mu+1}      & \hat{h}_{P-s-1+\mu+1} & \dots & \hat{h}_{P-3+\mu+1} \\
				\hat{h}_{P-s-3+\mu+1}       & \hat{h}_{P-s-2+\mu+1} & \dots & \hat{h}_{P-4+\mu+1} \\
				\hdots & \hdots & \ddots &\vdots \\
				\hat{h}_{P-2s+\mu+1}  & \hat{h}_{P-s+1+\mu+1} & \dots & \hat{h}_{P-s+1+\mu+1}
			\end{bmatrix}
			\hat{\beta}
		\end{matrix}	
		= \begin{matrix}
			\begin{bmatrix}
				\hat{h}_{P-s-1+\mu}      & \hat{h}_{P-s-2+\mu} & \dots & \hat{h}_{P-2+\mu} \\
				\hat{h}_{P-s-2+\mu}       & \hat{h}_{P-s-1+\mu} & \dots & \hat{h}_{P-3+\mu} \\
				\hdots & \hdots & \ddots &\vdots \\
				\hat{h}_{P-(2s-1)+\mu}  & \hat{h}_{P-s+\mu} & \dots & \hat{h}_{P-s+\mu}
			\end{bmatrix}
		\hat{\beta}
		\end{matrix}\\
		=&
			\begin{matrix}
			\begin{bmatrix}
				\hat{h}_{P-1+\mu}\\
				\hat{h}_{P-2+\mu}\\
				\vdots\\
				\hat{h}_{P-(s-1)+\mu}
			\end{bmatrix}
		\end{matrix}
		=
		\begin{matrix}
			\begin{bmatrix}
				\hat{h}_{P-2+(\mu+1)}\\
				\hat{h}_{P-3+(\mu+1)}\\
				\vdots\\
				\hat{h}_{P-s+(\mu+1)}
			\end{bmatrix}
		\end{matrix}.		
	\end{align*}	
	Combing those two result we have proved 
	\begin{align*}
		\begin{matrix}
			\begin{bmatrix}
				\hat{h}_{P-s-1+\mu+1}      & \hat{h}_{P-s+\mu+1} & \dots & \hat{h}_{P-2+\mu+1} \\
				\hat{h}_{P-s-2+\mu+1}       & \hat{h}_{P-s-1+\mu+1} & \dots & \hat{h}_{P-3+\mu+1} \\
				\hdots & \hdots & \ddots &\vdots \\
				\hat{h}_{P-2s+\mu+1}  & \hat{h}_{P-s+1+\mu+1} & \dots & \hat{h}_{P-s-1+\mu+1}
			\end{bmatrix}
			\hat{\beta}
			=
			\begin{bmatrix}
				\hat{h}_{P-1+\mu+1}\\
				\hat{h}_{P-2+\mu+1}\\
				\vdots\\
				\hat{h}_{P-s+\mu+1}
			\end{bmatrix}
		\end{matrix}.
	\end{align*}	
	By induction, the equation \ref{eq:CycInduction} holds for all $\ell=0,1,\cdots,P-1$.
	Equation \ref{eq:CycInduction} immediately finishes the proof.	
\end{proof}

Now we are ready to prove Lemma \ref{Thm:CycCodeDetection}.
By Lemma \ref{Lemma:CycProb}, for the $P$-dimensional vector $\gamma = \left(\mathbf{f} \mathbf{N} \right)^T$, we have  
\begin{align*}
	\textit{Pr}( \{ j : \gamma_j \not= 0 \} = \{j: \| \mathbf{N}_{\cdot,j} \|_0 \not= 0 \}) = 1,
\end{align*}

Since there are at most $s$ adversaries, the number of non-zero columns in $\mathbf{N}$ is at most $s$ and hence there are at most $s$ non-zero elements in $\gamma$, i.e., $\|\gamma\|_0\leq s$, with probability 1.
Now consider the case when $\|\gamma\|_0\leq s$.
First note that $[h_{P-2s}, h_{P-2s+1},\cdots, h_{P-1}] = \mathbf{f} \mathbf{R}^{Cyc} \mathbf{C}_{R}^{\dag} = \mathbf{f} \mathbf{N} \mathbf{C}_{R}^{\dag} = \gamma^T \mathbf{C}_{R}^{\dag}$, where the second equation is due to Lemma \ref{Lemma:Orthogonal}.
Now let us construct $\hat{\mathbf{h}} = [\hat{h}_0,\hat{h}_1,\cdots, \hat{h}_{P-1}]^T$ by $\hat{\mathbf{h}} = \mathbf{C}^{\dag} \gamma$.
Note that $\mathbf{C}$ is symmetric and thus $\mathbf{C}^{\dag} = \left[\mathbf{C}^{\dag}\right]^T$. One can easily verify that $\hat{h}_{\ell} = {h}_{\ell},\ell = P-2s, P-2s+1,\cdots,P-1$.
Therefore, the equation 
\begin{align*}
	\begin{matrix}
		\begin{bmatrix}
			h_{P-s-1}      & h_{P-s} & \dots & h_{P-2} \\
			h_{P-s-2}       & h_{P-s-1} & \dots & h_{P-3} \\
			\hdots & \hdots & \ddots &\vdots \\
			h_{P-2s}  & h_{P-s+1} & \dots & h_{P-s+1}
		\end{bmatrix}
		\begin{bmatrix}
			\beta_{0}\\
			\beta_{1}\\
			\vdots\\
			\beta_{s-1}
		\end{bmatrix}
		=
		\begin{bmatrix}
			h_{P-1}\\
			h_{P-2}\\
			\vdots\\
			h_{P-s}
		\end{bmatrix}
	\end{matrix}
\end{align*}
becomes 
\begin{align*}
	\begin{matrix}
		\begin{bmatrix}
			\hat{h}_{P-s-1}      & h_{P-s} & \dots & \hat{h}_{P-2} \\
			\hat{h}_{P-s-2}       & \hat{h}_{P-s-1} & \dots & \hat{h}_{P-3} \\
			\hdots & \hdots & \ddots &\vdots \\
			\hat{h}_{P-2s}  & \hat{h}_{P-s+1} & \dots & \hat{h}_{P-s+1}
		\end{bmatrix}
		\begin{bmatrix}
			\beta_{0}\\
			\beta_{1}\\
			\vdots\\
			\beta_{s-1}
		\end{bmatrix}
		=
		\begin{bmatrix}
			\hat{h}_{P-1}\\
			\hat{h}_{P-2}\\
			\vdots\\
			\hat{h}_{P-s}
		\end{bmatrix}
	\end{matrix}
\end{align*}
which always has a solution.
Assume we find one solution $\bar{\beta} = [\bar{\beta}_0, \bar{\beta}_1,\cdots, \bar{\beta}_{P-1}]^T$.
By the second part of Lemma \ref{Lemma:DFT}, we have 
\begin{align*}
	\hat{h}_{\ell} = \sum_{u=0}^{s-1} \bar{\beta}_{u} \hat{h}_{\ell+u-s}, \forall \ell. 
\end{align*}
Now we prove by induction that $h_\ell = \hat{h}_\ell, \ell=0,1,\cdots, P-1$.

When $\ell=0$, we have 
\begin{align*}
	\hat{h}_{0} = \sum_{u=0}^{s-1} \bar{\beta}_{u} \hat{h}_{u-s} = \sum_{u=0}^{s-1} \bar{\beta}_{u} {h}_{u-s} = h_0
\end{align*}
where the second equation is due to the fact that  	$[h_{P-2s}, h_{P-2s-1},\cdots, h_{P-1}] = [\hat{h}_{P-2s}, \hat{h}_{P-2s-1},\cdots, \hat{h}_{P-1}]$ and $\hat{h}_{P+\ell} = \hat{h}_{\ell}, {h}_{P+\ell} = {h}_{\ell}$ (by definition).

Assume that for $\ell \leq\mu$, $\hat{h}_\ell = h_\ell$. 

When $\ell=\mu+1$, we have
\begin{align*}
	\hat{h}_{\mu+1} = \sum_{u=0}^{s-1} \bar{\beta}_{u} \hat{h}_{\mu+1+u-s} = \sum_{u=0}^{s-1} \bar{\beta}_{u} {h}_{\mu+1+u-s} = h_{\mu+1}
\end{align*}
where the second equation is because of the induction assumption for $\ell \leq\mu$, $\hat{h}_\ell = h_\ell$.

Thus, we have $h_\ell=\hat{h}_\ell$ for all $\ell$, which means $\mathbf{h} = \hat{\mathbf{h}} = \mathbf{C}^{\dag} \gamma$.
Thus $\mathbf{t}$,  the IDFT of $\mathbf{h}$, becomes $\mathbf{t} = \mathbf{C} \mathbf{h} = \mathbf{C} \mathbf{C}^{\dag} \gamma = \gamma$. 
Then the returned Index Set $V=\{j:e_{j+1}\not=0\} = \{j:\gamma_{j}\not=0\}$.
By Lemma \ref{Lemma:CycProb}, with probability 1, 
$\{j:\gamma_{j}\not=0\} = \{j:\|\mathbf{n}_{j}\|_0\not=0\}$.
Therefore, we have with probability 1, 
$V = \{j:\|\mathbf{n}_{j}\|_0\not=0\}$, which finishes the proof. \qed

\subsection{Proof of Theorem \ref{Thm:CycCodeOptimality} }
We first prove the correctness of the cyclic code.
By Lemma \ref{Thm:CycCodeDetection}, the set $U$ contains the index of all non-adversarial compute nodes with probability 1.
By Lemma \ref{Thm:CycCodeColumnSpan}, there exists $\mathbf{b}$ such that $\mathbf{W}_{\cdot,U} b = \mathbf{1}$. 
Therefore, $\mathbf{u}^{Cyc} = \mathbf{R}_{\cdot, U}^{Cyc} \mathbf{b} =  (\mathbf{G} \mathbf{W}+\mathbf{N})_{\cdot,U} \mathbf{b} = \mathbf{G}\mathbf{W}_{\cdot,U} \mathbf{b} = \mathbf{G} \mathbf{1}_{P}$. 
Thus,  The cyclic code  $(\mathbf{A}^{\mathit{Cyc}}, E^{\mathit{Cyc}}, D^{Cyc})$ can recover the desired gradient and hence resist any $\leq s$ adversaries with probability 1.

Next we show the efficiency of the cyclic code.
By the construction of $\mathbf{A}^{Cyc}$ and $\mathbf{W}$, the redundancy ratio is $2s+1$ which reaches the lower bound.
Each compute node needs to compute a linear combination of the gradients of the data it holds, which needs $\mathcal{O}((2s+1)d)$ computations.
For the PS, the detection function $\phi(\cdot)$ takes $\mathcal{O}(d)$ (generating the random vector $\mathbf{f}$) + $\mathcal{O}(dP + 2 P s )$ (computing $\mathbf{f} \mathbf{R} \mathbf{C}_{R}^{\dag}$) + $\mathcal{O}(s^2)$ (solving the Toeplitz system of linear equations in \eqref{eq:RecoverLinearSystem} ) + $\mathcal{O}((P-2s)s)$ (computing $h_{\ell}, \ell=0,1,2,\cdots, P-2s-1$ ) +  $\mathcal{O}(P\log P)$ (computing the DFT of $\mathbf{h}$ ) + $\mathcal{O}(P)$ (examining the non-zero elements of $\mathbf{t}$ ) = $\mathcal{O}( d + dP + 2 P s +s^2 +(P-2s)s+ P\log P+P) = \mathcal{O}( dP + P s + P\log P) $.
Finding the vector $\mathbf{b}$ takes $\mathcal{O}(P^3)$ (by simply constructing $\mathbf{b}$ via $\left[\mathbf{C}_{L}\right]_{\cdot,\bar{U}}$, though better algorithms may exist).
The recovering equation $\mathbf{R}_{\cdot,U} \mathbf{b}$ takes $O(dP)$.
Thus, in total, the decoder at the PS takes $\mathcal{O}(dP + P^3 + P\log P)$.
When $d\gg P$, \ie $d = \Omega(P^2)$, this becomes $\mathcal{O}(dP)$.
Therefore, $(\mathbf{A}^{\mathit{Cyc}}, E^{\mathit{Cyc}}, D^{Cyc})$ also achieves linear-time encoding and decoding. \qed

\section{Streaming Majority Vote Algorithm}
In this section we present the Boyer–-Moore majority vote algorithm \cite{MajorityVote}, which is an algorithm that only needs computation linear in the size of the sequence.

\begin{algorithm}[H]
	\SetKwInOut{Input}{Input}
	\SetKwInOut{Output}{Output}
	\Input{$n$ items $I_1, I_2,\cdots, I_n$}
	\Output{The majority of the $n$ items}
	Initialize an element $\textit{Ma} = I_1$ and a counter $\textit{Counter} = 0$.\\
    \For{$i=1$ \KwTo $n$ }
    {
        \uIf{$\textit{Counter}==0$}
        {
             $Ma = I_i$.\\
             $\textit{Counter} = 1$.
        }
        \uElseIf{$Ma == I_i$}
        {
            $\textit{Counter} = \textit{Counter} + 1$.
        }
        \Else
        {
            $\textit{Counter} = \textit{Counter} - 1$.
        }
  }
  Return $Ma$.
	\caption{Streaming Majority Vote.}
	\label{Alg:MajorityVote.}
\end{algorithm}
Clearly this algorithm runs in linear time and it is known that if there is a majority item then the algorithm finally will return it \cite{MajorityVote}.

\end{document}